\documentclass[letterpaper]{article} 
\usepackage{aaai25}
\usepackage{times}  
\usepackage{helvet}  
\usepackage{courier}  
\usepackage[hyphens]{url}  
\usepackage{graphicx} 
\usepackage{natbib}  
\usepackage{caption} 
\author {
    Qi Shi\textsuperscript{\rm 1},
    Pavel Naumov\textsuperscript{\rm 1}
}
\affiliations{
    \textsuperscript{\rm 1}University of Southampton, UK\\
    qi.shi@soton.ac.uk; p.naumov@soton.ac.uk
}

\usepackage{enumitem}
\usepackage{amsmath}
\usepackage{amssymb}
\usepackage{booktabs}
\usepackage{xcolor}
\usepackage[ruled,vlined,linesnumbered]{algorithm2e}
\makeatletter
\newcounter{pseudocode}

\makeatother

\newenvironment{proof}{\begin{trivlist}\item\noindent{\em Proof.}}{\hfill {\footnotesize$\square$}\end{trivlist}}

\newenvironment{proof-of-claim}{\begin{trivlist}\item\noindent{\em Proof of Claim.}}{\hfill {\tiny $\boxtimes$}\end{trivlist}}

\newtheorem{theorem}{Theorem}

\newtheorem{lemma}{Lemma}
\newtheorem{definition}{Definition}
\newtheorem{corollary}{Corollary}
\newtheorem{claim}{Claim}
\newtheorem{proposition}{Proposition}

\newlist{Sample}{enumerate}{1}
\setlist[Sample,1]{label={\em (\roman*)}, ref=(\em \roman*),left=4pt}

\usepackage[labelformat=simple]{subcaption} 

\newcommand{\nlhd}{\mbox{$\,\not\!\!\lhd\,$}}
\renewcommand{\hline}{\noindent\textcolor{red}{=========================================}}

\newcommand{\s}{\boldsymbol{s}}

\renewcommand{\epsilon}{\varepsilon}

\DeclareFontFamily{U}{mathb}{}
\DeclareFontSubstitution{U}{mathb}{m}{n}
\DeclareFontShape{U}{mathb}{m}{n}{
  <-5.5> mathb5
  <5.5-6.5> mathb6
  <6.5-7.5> mathb7
  <7.5-8.5> mathb8
  <8.5-9.5> mathb9
  <9.5-11> mathb10
  <11-> mathb12
}{}

\DeclareRobustCommand{\sqsubseteq}{\,\text{{\usefont{U}{mathb}{m}{n}\symbol{"84}}}\,}

\DeclareRobustCommand{\sqsupseteq}{\,\text{{\usefont{U}{mathb}{m}{n}\symbol{"85}}}\,}

\DeclareRobustCommand{\sqsupsetneq}{\,\text{{\usefont{U}{mathb}{m}{n}\symbol{"89}}}\,}

\title{Uncommon Belief in Rationality}

\begin{document}

\maketitle

\begin{abstract}
Common knowledge/belief in rationality is the traditional standard assumption in analysing interaction among agents. This paper proposes a graph-based language for capturing significantly more complicated structures of higher-order beliefs that agents might have about the rationality of the other agents. The two main contributions are a solution concept that captures the reasoning process based on a given belief structure and an efficient algorithm for compressing any belief structure into a unique minimal form.
\end{abstract}




\section{Introduction}\label{sec:introduction}

In the orthodox studies of game theory and game-modelled multiagent systems, the rationality of agents is usually assumed to be {\em common knowledge} \cite{aumann1976agreeing}.
Albeit called ``knowledge'', it does not have to be the case.
This is because, as defined in epistemology, knowledge is something that must be true and justifiable \cite{steup2024epistemology}.
However, from the perspective of each single agent, it is hard to verify that the other agents are indeed rational.
As discussed by \citet{lewis1969convention}, what really matters in the reasoning process is the agents' rationality and {\bf\em belief} about the other agents' rationality, the latter of which, unlike {\em knowledge}, is not necessarily true or justifiable.
Indeed, in the discussion of epistemic game theory \cite{dekel2015epistemic}, the assumption of {\bf\em rationality and common belief in rationality (RCBR)} serves as the foundation of the major solution concepts such as Nash equilibrium \cite{nash1950equilibrium}, correlated equilibrium \cite{aumann1987correlated,brandenburger1987rationalizability}, and rationalisability \cite{pearce1984rationalizable,bernheim1984rationalizable}.

As \citet{lewis1969convention} and \citet{schiffer1972meaning} interpret, RCBR consists of the rationality of all agents and a belief hierarchy that contains all finite sequences in the form that ``$a_1$ believes that $a_2$ believes \dots\ that $a_{i-1}$ believes that $a_i$ is rational'', where $a_1$, $a_2$, $\dots$, $a_{i-1}$, $a_i$ are (possibly duplicated) agents.
Observe that RCBR can be expressed with a complete digraph. 
For instance, Figure~\ref{fig:3AgentCommonBelief} illustrates the RCBR among agents $a, b, c$, where each node represents the rationality of an agent and each path\footnote{A path refers to a nonempty sequence of connected and possibly duplicated nodes. Hence, a single node forms a singleton path.} of at least two nodes corresponds to a sequence in the belief hierarchy of RCBR. 
Specifically, in Figure~\ref{fig:3AgentCommonBelief}, the node labelled with $a$ corresponds to the rationality of agent $a$; the path labelled with $(b,c,b,a)$ corresponds to the belief sequence ``agent $b$ believes that agent $c$ believes that agent $b$ believes that agent $a$ is rational''.
We denote this belief sequence by the tuple $(b,c,b,a)$ henceforth.
Notice that no self-loop exists due to the assumption that {\em agents do not have introspective beliefs} about their own rationality \cite{dekel2015epistemic}.

\begin{figure}
    \centering
    \begin{subfigure}[b]{0.11\textwidth}
        \begin{center}
        \scalebox{0.5}{\includegraphics{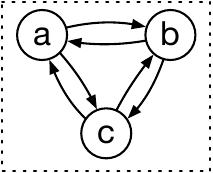}}
        \caption{RCBR}
       \label{fig:3AgentCommonBelief}
        \end{center}
    \end{subfigure}
        \begin{subfigure}[b]{0.05\textwidth}
        \begin{center}
        \scalebox{0.5}{\includegraphics{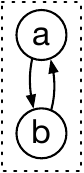}}
        \caption{}
       \label{fig:3Agent2CommonBelief}
        \end{center}
    \end{subfigure}
    \begin{subfigure}[b]{0.11\textwidth}
        \begin{center}
        \scalebox{0.5}{\includegraphics{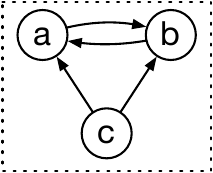}}
        \caption{}
        \label{fig:3AgentUncommonBelief}
        \end{center}
    \end{subfigure}
    \begin{subfigure}[b]{0.15\textwidth}
        \begin{center}
        \scalebox{0.5}{\includegraphics{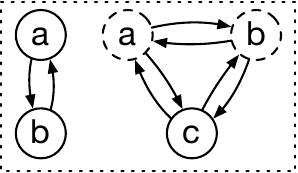}}
        \caption{}
        \label{fig:3AgentUncommonBelief-doxastic}
        \end{center}
    \end{subfigure}
    \caption{Different RBR graphs among agents $a,b,c$.}
    \label{fig:3AgentBelief}
\end{figure}

Yet, the RCBR assumption is too strong, especially in a system consisting of different types of agents.
For example, when two adults $a$ and $b$ interact with a child $c$, who is too young to possess rationality (\textit{i.e.} the ability to do mathematical optimisation), the RCBR assumption is satisfied only between $a$ and $b$. 
Then, the belief hierarchy can be captured by Figure~\ref{fig:3Agent2CommonBelief}.
Notice that, the irrational agent $c$ is not included in the graph because of the assumption that {\em irrational agents act arbitrarily and do not possess beliefs about the other agents' rationality}.\footnote{Even if irrational agents possess beliefs about the other agents' rationality, such beliefs do not affect the strategic behaviours in the system because irrational agents always act arbitrarily. Thus, it is safe to ignore irrational agents' beliefs about rationality when studying the strategic behaviours in a multiagent system.}
This assumption also implies that a belief hierarchy about rationality should be ``prefix-closed''. 
In other words, if agent $a$ believes that agent $b$ believes that agent $c$ is rational, then agent $a$ must believe that agent $b$ is rational, the latter of which further implies that agent $a$ must be rational.
This property makes it possible to illustrate a belief hierarchy with a directed graph.

Now we consider the situation that the child $c$, although young, is a genius. The adults $a$ and $b$ neglect the talent of $c$. But $c$ not only is rational and believes in the RCBR between $a$ and $b$, but also notices the arrogance of $a$ and $b$. 
In this case, there is no RCBR among agents $a,b,c$, but there is more than RCBR between agents $a$ and $b$.
We illustrate the {\bf\em rationality and beliefs in rationality (RBR)} in this system in Figure~\ref{fig:3AgentUncommonBelief}. 
We shall call it the ``RBR graph'' among agents $a$, $b$, $c$.
In this system, compared with the RCBR among $a,b,c$, there is no belief about the rationality of agent $c$. Meanwhile, compared with the RCBR between $a$ and $b$, there are agent $c$'s beliefs about the rationality and beliefs of $a$ and $b$.

Let us now consider another situation that, in the above case, the genius child $c$ fails to identify the arrogance of $a$ and $b$. Then, in $c$'s belief, the RCBR among $a,b,c$ still exists, which, however, is not the reality.
In other words, the agents $a$ and $b$ in $c$'s belief are not the real agents $a$ and $b$ in the system. We call such agents by {\bf\em doxastic agents}, who exist in beliefs but are not real.
In this case, we visualise the doxastic agents with dashed cycles as Figure~\ref{fig:3AgentUncommonBelief-doxastic} shows.
In this figure, the solid nodes labelled with $a$ and $b$ represent the real agents $a$ and $b$ between whom the RCBR exists, while the dashed nodes labelled with $a$ and $b$ represent the doxastic agents $a$ and $b$, with whom agent $c$ believes the RCBR exists.
In this case, only the solid nodes represent the rationality of agents; only the paths starting at solid nodes (\textit{i.e.} real agents) correspond to belief sequences in the system.

In this paper, we encode RBR systems with {\em directed labelled graphs}, as shown in Figure~\ref{fig:3AgentBelief}, based on which we study the agents' strategic behaviours.
Notably, the existing literature on agents' beliefs concentrates on modelling the Bayesian beliefs of agents (\textit{i.e.} agents' subjective probability distributions over the other agents' behaviours). 
On the contrary, the rationality of agents implies a dynamic ``best-response'' process \cite{savage1954foundation} that cannot be captured by static distributions.
Particularly, the {\em type structure} \cite{harsanyi1967games,brandenburger1993hierarchies} is used in epistemic game theory to implicitly model the hierarchy of Bayesian beliefs.
It is a recursive formalisation of strategic behaviours based on the RCBR assumption, rather than a formalisation of beliefs in rationality.
Meanwhile, the {\em influence diagram} \cite{koller2003multi,howard2005influence}, a variant of the Bayesian network \cite{pearl1988probabilistic}, is used in AI to model a decision system of agents.
Due to its strength in modelling the probabilistic uncertainty, the influence diagram is also used to model Bayesian beliefs of agents \cite{milch2000probabilistic}.
However, the acyclic nature of the influence graph makes it impossible to model RCBR as the RBR graph does.
More literature on modelling agents' beliefs can be seen in the review papers \cite{albrecht2018autonomous} and \cite{doshi2020recursively}.
As far as we know, no research on explicitly modelling uncommon RBR (\textit{i.e.} not RCBR) of agents exists in the literature.

{\bf Contribution and Outline}
We first discuss how RBR works in agents' strategic reasoning process in Section \ref{sec:rationality}.
Then, in section \ref{sec:terminologies}, we formally define {\em RBR graphs}, a graph-based language to capture uncommon RBR systems, and {\em (iterative) rationalisation}, the strategic reasoning process of agents, based on which we propose {\em doxastic rationalisability}, a solution concept in uncommon RBR systems.
After this, we discuss the equivalence of two RBR graphs in Section \ref{sec:equivalence RBR graphs}, based on which we design an algorithm that can compress any RBR graph to a minimal equivalent form in Section \ref{sec:minimisation RBR graph}.
Due to page limitation, we retain the formal definitions, theorems, and some informal discussions that capture intuition in the main text, while placing all supporting lemmas and tedious formal proofs in the appendix for reference.

\section{Rationality and Rationalisation}\label{sec:rationality}

As defined by \citet{savage1954foundation}, a rational agent, when faced with uncertainty, first forms a \textit{subjective} probability distribution over all possibilities, and then chooses a strategy that {\em best responds} (\textit{i.e.} maximises the expected utility) to the subjective probability distribution.
This is the commonly acknowledged definition of {\bf\em rationality} in economics, game theory, and multiagent systems.
{\bf\em Rationalisability} \cite{pearce1984rationalizable,bernheim1984rationalizable} is a solution concept to the question that ``what RCBR exactly implies''.
In this concept, a strategy is {\em rationalisable} if it best responds to some subjective probability distribution. 
Rational agents take only rationalisable strategies.
We call the process of finding rationalisable strategies {\bf\em rationalisation}.
It is a well-known result in game theory that, in a game with compact strategy sets and continuous utility functions, a (mixed) strategy is rationalisable if and only if it is not {\em strongly dominated}\footnote{A strategy is strongly dominated if it is always worse than another strategy. Coordinately, a strategy is weakly dominated if it is never better and sometimes worse than another strategy. They are standard terminologies in game theory.} \cite{gale1950solutions,pearce1984rationalizable}.
In this situation, rationalisation exactly means the elimination of strongly dominated strategies.
In other cases, the elimination of strongly dominated strategies implies\footnote{In the sense that a strongly dominated strategy is never rationalisable \cite{pearce1984rationalizable}.} rationalisation but not the other way around.
For example, \citet{borgers1993pure} finds that, if the preference of agents on outcomes is a total order rather than defined as utility functions, then rationalisation is a concept weaker than ``eliminating strongly dominated strategies'' but stronger than ``eliminating weakly dominated strategies''.

Since our purpose in this paper is not to discuss the essence of rationality, we simplify the definition of rationalisation to {\bf\em the elimination of strongly dominated strategies in pure strategy space}.
We say that a rational agent chooses only pure strategies that are {\em not} strongly dominated by any pure strategy (see Section~\ref{sec:terminologies} for formal definitions). 
The benefit is threefold: first, it avoids the computational intractability in dealing with mixed strategy space and ``best response'' optimisation; second, it allows us to discuss a more general game frame where the preference of agents is just partial order; third, the simplified definition is more restrictive (but not too much) than those in the literature, so the technical results of this paper can potentially be extended to more general definitions using similar proof techniques but at the expense of mathematical complexity.
Henceforth, by saying a strategy is dominated, we mean it is strongly dominated.

To see how rationality interacts with belief, let us consider the following simplified ``guess $2/3$ of the average'' game \cite{moulin1986game,nagel1995unraveling}:
\begin{quote}
\textit{Each agent chooses an \textbf{integer} in the interval $[1,10]$. The one whose choice is the closest to $2/3$ of the average of the \textbf{other} agents' choices wins.}
\end{quote}
Informally, we say that $2/3$ of the average of the other agents' choices is the {\em target} which every agent in this game aims to approach.
Suppose the agents $a,b,c$ as described in Section~\ref{sec:introduction} play the above game.
We analyse the potential choices of each agent in each of the four RBR systems illustrated with the RBR graphs in Figure~\ref{fig:3AgentBelief}.

We first consider the orthodox case that RCBR exists among agents $a,b,c$, as depicted in Figure~\ref{fig:3AgentCommonBelief}.
According to the description of the game, agent $a$ knows that the choices of agents $b$ and $c$, denoted by $\chi_b$ and $\chi_c$ henceforth, lie in the interval $[1,10]$.
Then, the target $t_a$ of agent $a$ lies in the interval $[\frac{2}{3}\cdot\frac{1+1}{2},\frac{2}{3}\cdot\frac{10+10}{2}]=[\frac{2}{3},{6\frac{2}{3}}]$.
Thus, $7$ is always closer to the target than any integer greater than it (\textit{i.e.} $8,9,10$ are dominated by $7$), while every integer in the interval $[1,7]$ might be the closest to the target.
Hence, the rationalisable choice $\chi_a$ of agent $a$ should be in the interval $[1,7]$.
Note that, the above analysis works for agents $b$ and $c$ in a symmetric way.
Therefore, the rationalisable choices are $\chi_a,\chi_b,\chi_c\in[1,7]$ after one rationalisation.
We mark this observation in column $1^{st}$ of lines (1-3) in Table~\ref{tab:rationalisation}.

Note that, the RCBR assumption implies the belief sequences $(a,b)$ and $(a,c)$, which represent that ``agent $a$ believes that agent $b$ is rational'' and ``agent $a$ believes that agent $c$ is rational'', respectively.
Then, agent $a$, after the above analysis, believes the choices $\chi_b,\chi_c\in[1,7]$.
Thus, a preciser target of agent $a$ becomes $t_a\in[\frac{2}{3},{4\frac{2}{3}}]$.
Hence, the rationalisable choice of agent $a$ is $\chi_a\in[1,5]$ after the second rationalisation. 
Similarly, $\chi_b\in[1,5]$ follows from the belief sequences $(b,a)$ and $(b,c)$ and $\chi_c\in[1,5]$ follows from the belief sequences $(c,a)$ and $(c,a)$.
We mark this observation in column $2^{nd}$ of lines (1-3) in Table~\ref{tab:rationalisation}.
Observe that, the RCBR assumption also implies the belief sequences $(a,b,c)$, $(a,b,a)$, $(a,c,b)$, and $(a,c,a)$.
Then, agent $a$, after the above analysis, believes the choices $\chi_b,\chi_c\in[1,5]$.
Hence, $\chi_a\in[1,3]$ after the third rationalisation, similar for $\chi_b$ and $\chi_c$.
We mark this observation in column $3^{rd}$ of lines (1-3) in Table~\ref{tab:rationalisation}.
Following the same process, after the fourth rationalisation, $\chi_a,\chi_b,\chi_c\in[1,2]$ and, after the fifth rationalisation, $\chi_a,\chi_b,\chi_c=1$.
Then, more rationalisation will not eliminate more strategies.
We mark these observations in the corresponding columns of lines (1-3) in Table~\ref{tab:rationalisation}.
Consequently, given the RCBR assumption, every agent should choose $1$.
The above analysis follows the standard approach in epistemic game theory \cite[Example 3.7]{perea2012epistemic}.

\begin{table}[htb]
\centering
\footnotesize
\renewcommand\arraystretch{1.05}
\setlength{\tabcolsep}{5pt}
\begin{tabular}{ccccccccc}
& & $1^{st}$ & $2^{nd}$ & $3^{rd}$ & $4^{th}$ & $5^{th}$ & $6^{th}$ & $\dots$ \\
\cmidrule[.5pt]{2-9}
(1) & $a$ & $[1,7]$ & $[1,5]$ & $[1,3]$ & $[1,2]$ & $1$ & $1$ & $\dots$ \\
(2) & $b$ & $[1,7]$ & $[1,5]$ & $[1,3]$ & $[1,2]$ & $1$ & $1$ & $\dots$ \\
(3) & $c$ & $[1,7]$ & $[1,5]$ & $[1,3]$ & $[1,2]$ & $1$ & $1$ & $\dots$ \\ 
\cmidrule[.5pt]{2-9}
(4) & $a$ & $[1,7]$ & $[1,6]$ & $[1,5]$ & $[1,5]$ & $\dots$  \\
(5) & $b$ & $[1,7]$ & $[1,6]$ & $[1,5]$ & $[1,5]$ & $\dots$  \\
(6) & $c$ & \multicolumn{2}{l}{$[1,10] \hspace{3mm}\dots$}  \\
\cmidrule[.5pt]{2-9}
(7) & $c$ & $[1,7]$ & $[1,5]$ & $[1,4]$ & $[1,3]$ & $[1,3]$ & $\dots$ \\
\end{tabular}
\caption{Rationalisations on the RBR graphs in Figure~\ref{fig:3AgentBelief}.}
\label{tab:rationalisation}
\end{table}

Now we look at the non-trivial cases where RCBR does not exist.
Suppose the RBR among agents $a,b,c$ is as Figure~\ref{fig:3Agent2CommonBelief} shows.
Since agent $c$ is irrational, her choice is always $\chi_c\in[1,10]$, as depicted in line (6) of Table~\ref{tab:rationalisation}.
For agents $a$ and $b$, the first rationalisation is identical to the case with RCBR.
That is $\chi_a,\chi_b\in[1,7]$, as depicted in column $1^{st}$ of lines (4-5) in Table~\ref{tab:rationalisation}.
Next, note that the belief sequence $(a,b)$ is in the RBR system but $(a,c)$ is not.
Then, after the above analysis, agent $a$ believes that $\chi_b\in[1,7]$ and $\chi_c\in[1,10]$.
Thus, the target of agent $a$ is $t_a\in[\frac{2}{3},5\frac{2}{3}]$.
Hence, the rationalisable choice of agent $a$ is $\chi_a\in[1,6]$ after the second rationalisation.
Symmetrically, $\chi_b\in[1,6]$ follows from the belief sequence $(b,a)$ after the second rationalisation.
These results are marked in column $2^{nd}$ of lines (4-5) in Table~\ref{tab:rationalisation}.
Then, due to the belief sequences $(a,b,a)$ and $(b,a,b)$, the above analysis implies that the preciser targets of agents $a$ and $b$ are $t_a,t_b\in[\frac{2}{3},5\frac{1}{3}]$.
Thus, $\chi_a,\chi_b\in[1,5]$ after the third rationalisation, as marked in column $3^{rd}$ of lines (4-5) in Table~\ref{tab:rationalisation}.
It can be verified that more rationalisation will not eliminate more strategies.
Consequently, the choices are $\chi_a,\chi_b\in[1,5]$ and $\chi_c\in[1,10]$ in the RBR system depicted in Figure~\ref{fig:3Agent2CommonBelief}.

As the above two cases show, in a given game, RBR implies an {\em iterative} rationalisation process until a stable state where no more strategy can be eliminated is reached.
In essence, the iterative process relies on longer and longer belief sequences.
However, since we depict an RBR system with a digraph, if there is an edge from a node labelled with $a$ to a node labelled with $b$, then, for each belief of agent $b$, agent $a$ believes that agent $b$ holds this belief.
For instance, in Figure~\ref{fig:3AgentCommonBelief}, for the agent $b$'s belief sequence $(b,c,a)$, agent $a$ holds the belief sequence $(a,b,c,a)$ and agent $c$ holds the belief sequence $(c,b,c,a)$.
This property allows us to {\em do iterative rationalisation without explicitly considering the belief sequences.}
Instead, if an edge from a node labelled with agent $a$ to a node labelled with agent $b$ exists in an RBR-graph, then agent $a$ does the $i^{th}$ rationalisation based on the result of agent $b$'s $(i-1)^{th}$ rationalisation.
In this sense, even if agent $b$ is a doxastic agent, she is treated the same way as a real agent.
For convenience of expression, we also refer to the belief sequences of doxastic agents as belief sequences in the RBR system.
This is in line with the model of {\em iterated strategic thinking} \cite{binmore1988modeling}, which formalises the intuition that ``the natural way of looking at game situations \dots\ is not based on circular concepts, but rather on a step-by-step reasoning procedure'' \cite{selten1998features}.

It is interesting to observe that, in the RBR system depicted in Figure~\ref{fig:3AgentUncommonBelief}, agents $a$ and $b$ hold the same belief hierarchy as in Figure~\ref{fig:3Agent2CommonBelief}.
In other words, agents $a$ and $b$ cannot distinguish the RBR systems in Figure~\ref{fig:3Agent2CommonBelief} and Figure~\ref{fig:3AgentUncommonBelief}.
As a result, the iterative rationalisation process of agents $a$ and $b$ runs identically in Figure~\ref{fig:3Agent2CommonBelief} and Figure~\ref{fig:3AgentUncommonBelief}, as lines (4-5) in Table~\ref{tab:rationalisation} show.
Contrarily, agent $c$ in Figure~\ref{fig:3AgentUncommonBelief},  different from agent $c$ in Figure~\ref{fig:3Agent2CommonBelief}, is rational and holds a belief hierarchy.
In the RBR system depicted in Figure~\ref{fig:3AgentUncommonBelief}, the first rationalisation of agent $c$, which relies only on the rationality of $c$, works in the same way as the case in Figure~\ref{fig:3AgentCommonBelief}.
From the second rationalisation, agent $c$ does the $i^{th}$ rationalisation based on the results of the $(i-1)^{th}$ rationalisation.
We record the process in line (7) of Table~\ref{tab:rationalisation}.
For example, in the fourth rationalisation, based on the analysis of the third rationalisation, agent $c$ believes that the choices $\chi_a,\chi_b\in[1,5]$, as marked in column $3^{nd}$ of lines (4-5) in Table~\ref{tab:rationalisation}. 
Thus, the target $t_c\in[\frac{2}{3},3\frac{1}{3}]$. 
Hence, the rationalisable choice is $\chi_c\in[1,3]$, as depicted in column $4^{th}$ of line (7) in Table~\ref{tab:rationalisation}. 
Also observe that, in the RBR system depicted in Figure~\ref{fig:3AgentUncommonBelief-doxastic}, agents $a$ and $b$ hold the same belief as in Figure~\ref{fig:3Agent2CommonBelief}, while agent $c$ holds the same belief as in Figure~\ref{fig:3AgentCommonBelief}.
As a result, the iterative rationalisation process in Figure~\ref{fig:3AgentUncommonBelief-doxastic} is as lines (3-5) of Table~\ref{tab:rationalisation} show.

It is worth mentioning that, our RBR-graph can be used to model {\em bounded rationality} \cite{simon1955behavioral,simon1957models} of agents.
In particular, the {\em cognitive hierarchy model} \cite{camerer2004cognitive,chong2016generalized} is a mathematical model of bounded rationality that limits the length of belief sequences in a belief hierarchy.
The model assumes every agent believes herself to be smarter than everyone else and thus performs deeper reasoning.
Particularly, an $i$-step reasoner believes that every other agent is a $j$-step reasoner with some probability, where $j<i$, and a $0$-step reasoner is an irrational agent.
Because of this assumption, the statement that ``an agent cannot reason too deeply'' is reduced to the statement that ``the agent believes that the other agents cannot reason too deeply''.
The former is the essence of bounded rationality, while the latter makes it possible to depict a bounded rationality RBR system with an acyclic RBR-graph.
Moreover, the iterative rationalisation process in bounded rationality RBR graphs is exactly the step-by-step reasoning process in the cognitive hierarchy model. 
However, the cognitive hierarchy model works implicitly in a probabilistic approach and is not used to model the explicit RBR systems as the RBR graphs do.

\section{Terminologies and Solution Concept}\label{sec:terminologies}

In preparation for the study of the iterative rationalisation process in games among agents with uncommon RBR, we next formalise the concepts informally introduced in the previous sections.
Throughout this paper, unless specified otherwise, we assume a fixed nonempty set $\mathcal{A}$ of agents.
We start with a general definition of (strategic) games that uses partial orders~\cite{osborne1994course}.
\begin{definition}\label{df:game}
A \textbf{game} is a tuple $(\Delta,\preceq)$ such that
\begin{enumerate}
    \item $\Delta=\{\Delta_a\}_{a\in\mathcal{A}}$, where $\Delta_a\ne\varnothing$ is a finite strategy space for each agent $a\in\mathcal{A}$; \label{dfitem:game strategy space}
    \item $\preceq\;=\{\preceq_a\}_{a\in\mathcal{A}}$, where $\preceq_a$ is a partial order on the Cartesian product $\prod_{b\in\mathcal{A}}\Delta_b$.\label{dfitem:game preference}
\end{enumerate}
\end{definition}
\noindent An element $s_a\in\Delta_a$ is called a \textit{strategy} of agent $a$. 
An {\em outcome} is a tuple $\s\in\prod_{a\in\mathcal{A}}\Delta_a$ consisting of a strategy for each agent $a\in\mathcal{A}$.
Binary relation $\preceq_a$ shows agent $a$'s preference over the outcomes. 
For two outcomes $\s$ and $\s'$, if $\s\preceq_a\s'$ and $\s'\npreceq_a\s$, then we write $\s\prec_a\s\,'$ and say that agent $a$ \textit{strictly prefers} outcome $\s\,'$ to $\s$.
For example, in our simplified ``2/3 game'' in Section~\ref{sec:rationality}, the strategy space of each agent is all integers in the interval $[1,10]$; an outcome is a collection of every agent's choice.
An agent strictly prefers the outcomes where her choice is closer to her target.

Note that, the commonly used definition of games where preference is defined via utility functions is a special case of Definition~\ref{df:game} in which preference $\preceq_a$ is a total order for each agent $a\in\mathcal{A}$.
In particular, $\s\preceq_a\s'$ if $u_a(\s)\leq u_a(\s')$, where $u_a(\s)$ and $u_a(\s')$ are utilities of agent $a$ toward outcome $\s$ and $\s'$, respectively.

\begin{definition}\label{df:reasoning scene}
For any agent $a\in\mathcal{A}$, a \textbf{reasoning scene} $\Theta_a$ in the game $(\Delta,\preceq)$ is a Cartesian product $\prod_{b\in\mathcal{A}\setminus\{a\}}\Theta_a^b$ such that $\varnothing\subsetneq \Theta_a^b\subseteq\Delta_{b}$ for each agent $b \neq a$.
\end{definition}

A reasoning scene describes a static context in which a rational agent rationalises.
It captures the uncertainty of an agent toward the other agents' strategies.
In other words, given that every other agent $b$ may choose a strategy from set $\Theta_a^b$, agent $a$ reasons about which strategies of her own are rationalisable.
Recall that, when agents $a,b,c$ with RBR in Figure~\ref{fig:3Agent2CommonBelief} play the simplified ``2/3 game'', in the third rationalisation, agent $a$ believes $\chi_b\in[1,6]$ and $\chi_c\in[1,10]$, as column $2^{nd}$ of lines (5-6) in Table~\ref{tab:rationalisation} shows.  
In this situation, we say that agent $a$ rationalises in the reasoning scene $\Theta_a$ such that $\Theta_a^b$ consists of all integers in the interval $[1,6]$ and $\Theta_a^c$ consists of all integers in the interval $[1,10]$.
A tuple $\s_{-a}\in\Theta_a$ is a combination of all agents' strategies except agent $a$ and thus $(\s_{-a},s_a)$ is an outcome.

The next definition formalises the notion of dominance as discussed in Section~\ref{sec:rationality}.

\begin{definition}\label{df:dominated strategy}
For a given reasoning scene $\Theta_a$ of agent $a$ and any strategies $s_a,s'_a\in\Delta_a$, strategy $s'_a$ \textbf{dominates} strategy $s_a$ if $(\s_{-a},s_a) \prec_a (\s_{-a},s_a')$ for each tuple $\s_{-a} \in \Theta_a$.
\end{definition}

We write $s_a\lhd_{\Theta_a} s'_a$ if strategy $s_a$ is dominated by strategy $s'_a$ in the reasoning scene $\Theta_a$.
Note that, dominance relation is asymmetric (\textit{i.e.} if $s_a\lhd_{\Theta_a}s'_a$, then $s'_a\nlhd_{\Theta_a}s_a$).

The next definition formalises the result of rationalisation in a given reasoning scene.
It is in line with our simplified definition of rationalisation (\textit{i.e.} the elimination of dominated strategies) as discussed in Section~\ref{sec:rationality}.

\begin{definition}\label{df:rational response}
In reasoning scene $\Theta_a$, the \textbf{rational response} of agent $a$ is a set of strategies
\begin{equation}\notag
\Re_a(\Theta_a):=\{s_a\in\Delta_a\,|\,\neg\exists s'_a\in\Delta_a (s_a\lhd_{\Theta_a} s'_a)\}.
\end{equation}
\end{definition}

In other words, the set $\Re_a(\Theta_a)$ consists of all rationalisable (\textit{i.e.} not dominated) strategies of agent $a$ in the reasoning scene $\Theta_a$. 
Note that, $\Re_a(\Theta_a)\neq\varnothing$ for any agent $a$ and any reasoning scene $\Theta_a$ due to the asymmetry of dominance relation $\lhd_{\Theta_a}$.
Next, we formally define RBR graphs.
In the rest of the paper,
notation $nEm$ is short for $(n,m)\in E$, notation $\ell_n$ is short for $\ell(n)$, and notation $\pi_a$ is short for $\pi(a)$. 

\begin{definition}\label{df:RBR graph}
An \textbf{RBR graph} is a tuple $(N,E,\ell,\pi)$ where:
\begin{enumerate}
\item $(N,E)$ is a finite directed graph with set $N$ of the nodes and set $E\subseteq N \times N$ of the directed edges;\label{dfitem:RBR graph frame}
\item $\ell: N\to\mathcal{A}$ is a labelling function such that for each node $n,m_1,m_2\in N$,\label{dfitem:RBR graph labelling function}
\begin{enumerate}
    \item if $nEm_1$, then $\ell_{n}\neq \ell_{m_1}$;\label{dfitem:RBR graph labelling function 1} 
    \item if $nEm_1$, $nEm_2$, and $\ell_{m_1}=\ell_{m_2}$, then $m_1=m_2$; \label{dfitem:different child label}
\end{enumerate}    
\item $\pi:\mathcal{A}\to N$ is a partial designating function such that for each agent $a$, if $\pi_a$ is defined, then $\ell_{\pi_a}=a$;\label{dfitem:RBR graph designating function} 
\item for each node $n\in N$, there is an agent $a\in\mathcal{A}$ and a path from node $\pi_a$ to $n$.\label{dfitem:RBR graph no irrelevant dummy node}
\end{enumerate}
\end{definition}

An RBR graph defined above represents an RBR system among all agents in set $\mathcal{A}$.
Note that the nodes in RBR graphs are not agents but just labelled by agents. This is because multiple nodes may represent the same agent when doxastic agents exist, as shown in Figure~\ref{fig:3AgentUncommonBelief-doxastic}.
In particular, item~\ref{dfitem:RBR graph frame} defines the finite digraph structure $(N,E)$ of an RBR graph.
Each node in set $N$ represents either a real agent or a doxastic agent.
A sequence of nodes connected by edges in set $E$ forms a path corresponding to a belief sequence in the RBR system.
Item~\ref{dfitem:RBR graph labelling function} defines the labelling function $\ell$ such that $\ell_n$ is the agent that node $n$ represents. 
Item~\ref{dfitem:RBR graph labelling function 1} formalises the assumption that agents do not have introspective beliefs. 
Item~\ref{dfitem:different child label} captures the intuition that an agent has only one identity in another agent's belief, thus preventing belief conflicts.
Item~\ref{dfitem:RBR graph designating function} defines the designating function $\pi$ on the set of all rational agents such that $\pi_a$ is the node representing the real agent $a$ (captured by the solid nodes in Figure~\ref{fig:3AgentBelief}). 
Note that, each node $n$ {\em not} in the {\em image} of function $\pi$ (captured by the dashed nodes in Figure~\ref{fig:3AgentBelief}) represents a doxastic agent $\ell_n$.
Item~\ref{dfitem:RBR graph no irrelevant dummy node} requires that each node in an RBR graph should be reachable from (\textit{i.e.} relevant to) a real agent so that every object satisfying Definition~\ref{df:RBR graph} represents an RBR system.

In an RBR graph $(N,E,\ell,\pi)$, a path $p=(n_1,\dots,n_k)$ where $k\geq 1$ is a sequence of $k$ nodes such that $n_iEn_{i+1}$ for each integer $i<k$.
We call the sequence $\sigma=(\ell_{n_1},\dots,\ell_{n_k})$ of agents the {\bf\em belief sequence} of path $p$ and read it as ``agent $\ell_{n_1}$ believes that \dots\ believes that agent $\ell_{n_k}$ is rational''.
For a finite sequence $\sigma$ of (possibly duplicated) agents, denote by $|\sigma|$ the length of sequence $\sigma$.
For any agent $a\in\mathcal{A}$ and any sequence $\sigma$ of agents, $a\!::\!\sigma$ is the sequence obtained by attaching agent $a$ at the beginning of sequence $\sigma$.
A sequence $\sigma$ is called an {\bf\em alternating sequence} if every two consecutive agents in $\sigma$ are not equal.

\begin{definition}\label{df:path set}
For each node $n$ in an RBR graph $(N,E,\ell,\pi)$, each integer $i\geq 1$, and each integer $j\geq 0$,
\begin{equation}\label{eq:path set main}
\Pi_n^i:=
\begin{cases}
\{\ell_n\}, &\text{if } i=1;\\
\{\ell_n\!::\!\sigma\mid nEm,\sigma\in \Pi_{m}^{i-1}\}, &\text{if }  i\geq 2;
\end{cases}
\end{equation}
\begin{equation}\label{eq:accumulated path set}
\Psi_n^j:=\bigcup_{0<i\leq j}\Pi_n^i;
\end{equation}
\begin{equation}\label{eq:whole path set}
\Psi_n^*:=\bigcup_{i>0}\Pi_n^i.
\end{equation}
\end{definition}

Informally, for a node $n$ in an RBR graph, set $\Pi_n^i$ consists of the belief sequences of all paths starting at node $n$ and of length $i$; set $\Psi_n^j$ consists of the belief sequences of all paths starting at node $n$ and of length at most $j$; set $\Psi_n^*$ consists of the belief sequences of all paths starting at node $n$.
Intuitively, $\Psi_n^*$ denotes the belief hierarchy of the (real or doxastic) agent represented by node $n$.

In the rest of this section, we consider solution concepts of games among agents with (possibly) uncommon RBR.

\begin{definition}\label{df:solution}
A \textbf{solution} $S$ of the game $(\Delta,\preceq)$ on the RBR graph $(N,E,\ell,\pi)$ is a family of sets $\{S_n\}_{n\in N}$ such that $\varnothing\subsetneq S_n\subseteq\Delta_{\ell_n}$ for each node $n\in N$.
\end{definition}

Informally, solution $S$ describes a type of uncertainty in the choice of each (real or doxastic) agent in an RBR system: the (real or doxastic) agent represented by node $n$ chooses only strategies in set $S_n$.
We denote the solution $\{S_n\}_{n\in N}$ by $S$ if it causes no ambiguity.
Specifically, let
\begin{equation}\label{eq:solution Delta}
S^{\Delta}:=\{\Delta_{\ell_n}\}_{n\in N}
\end{equation}
be the solution corresponding to the whole strategy space.

Note that a solution does not have to be ``reasonable''.
For instance, consider the ``2/3 game'' in Section~\ref{sec:rationality} and the RBR graph depicted in Figure~\ref{fig:3AgentUncommonBelief-doxastic}. A solution $S$ could be such that $S_n$ is the set of all {\em prime} integers in the interval $[1,10]$ for all five nodes $n$ in the RBR graph, which is obviously unreasonable.
Next, we consider the rationalisation of solutions.
To do this, we have the next auxiliary definition.

\begin{definition}\label{df:belief scene}
For any solution $S$ of the game $(\Delta,\preceq)$ on the RBR graph $(N,E,\ell,\pi)$, the \textbf{belief scene} $\Tilde{\Theta}_n(S)$ of any node $n\in N$ is the reasoning scene of agent $\ell_n$ such that for each agent $b\neq\ell_n$,
\begin{equation}\label{eq:RBR graph reasoning scene 1}
\Tilde{\Theta}_{n}^b(S):=
\begin{cases}
S_{n'},\!\! & \text{if } \exists n'\in N (nEn' \text{ and } \ell_{n'}=b); \\
\Delta_b,\!\! &  \text{otherwise.}
\end{cases}
\end{equation}
\end{definition}

Specifically, for any solution $S$ and any node $n$ in the RBR graph, consider the (real or doxastic) agent $\ell_n$ denoted by node $n$. 
For each agent $b\neq \ell_n$, if agent $\ell_n$ believes $b$ is rational, then there must be a node $n'$ labelled with $b$ such that $nEn'$ in the belief graph.
Moreover, node $n'$ captures the agent $b$ in agent $\ell_n$'s belief.
In this sense, given the solution $S$, agent $\ell_n$ believes that agent $b$ chooses a strategy from set $S_n$.
On the other hand, if agent $\ell_n$ believes $b$ is irrational, then no node $n'$ exists such that $\ell_{n'}=b$ and $nEn'$ and, in agent $\ell_n$'s belief, agent $b$ choose any strategy from set $\Delta_b$.
Hence, given solution $S$, the (real or doxastic) agent $\ell_n$ denoted by node $n$ believes that she rationalises in the reasoning scene in statement~\eqref{eq:RBR graph reasoning scene 1}.
We refer to such a reasoning scene as the belief scene of node $n$.
Then, the rationalisation of solution $S$ is such that every (real or doxastic) agent in an RBR graph rationalises (\textit{i.e.} rational response) in her belief scene, as formally defined below.

\begin{definition}\label{df:rationalisation on solution}
The \textbf{rationalisation} of any solution $S$ of the game $(\Delta,\preceq)$ on the RBR graph $(N,E,\ell,\pi)$ is the solution $\mathbb{R}(S)=\{\mathbb{R}(S)_n\}_{n\in N}$ such that $\mathbb{R}(S)_n:=\Re_{\ell_n}(\tilde{\Theta}_n(S))$ for each node $n\in N$.
\end{definition}

Note that, for any solution $S$, the rationalisation $\mathbb{R}(S)$ is a solution of the same game on the same RBR graph.
It captures one turn of the iterative rationalisation process as discussed in Section~\ref{sec:rationality}, which is formally defined below.

\begin{definition}\label{df:ith rationalisation}
The \textbf{$i^{th}$ rationalisation} on solution $S$ is
\begin{equation}\notag
\mathbb{R}^i(S):=\begin{cases}
    S, & i=0;\\
    \mathbb{R}(\mathbb{R}^{i-1}(S)), & i\geq 1.
\end{cases}
\end{equation}
\end{definition}

Recall that, the $i^{th}$ column of Table~\ref{tab:rationalisation} shows the result of the $i^{th}$ rationalisation for our ``2/3 game'' in the RBR systems denoted in Figure~\ref{fig:3AgentBelief}.
As shown there, the iterative rationalisation process may lead to a stable state. We call it {\em stable solution} and define it as follows.

\begin{definition} \label{df:stable solution}
A \textbf{stable solution} $S$ is such that $\mathbb{R}(S)=S$.
\end{definition}

Note that, without extra assumption (\textit{e.g.} the inaccessibility of some strategies), the iterative rationalisation should start at the solution $S^{\Delta}$ where every agent in an RBR graph chooses from the whole strategy space.
Moreover, without extra assumption (\textit{e.g.} the limitation of agents' mental capacity), the iterative rationalisation would continue forever because no agent has the incentive to stop it.
The definition below captures this idea.

\begin{definition}\label{df:rational solution}
The \textbf{rational solution} $\mathbb{S}$ is $\lim_{i\to\infty}\mathbb{R}^{i}(S^{\Delta})$.\footnote{Technically, $\mathbb{S}=\{\mathbb{S}_n\}_{n\in N}$ such that $\mathbb{S}_n=\lim_{i\to\infty}\mathbb{R}_n^{i}(S^{\Delta})$, where $\mathbb{R}_n^{i}(S^{\Delta})$ is a set of strategies for each integer $i$ and the limit of a sequence of sets is defined in the standard way~\cite[Section 1.3]{resnick1998probability}. In particular, if a sequence of sets stabilises after some element, then the limit is equal to the stable value.}
\end{definition}

However, as the following theorem shows, this process does not have to continue forever because a stable state will be achieved after finite iterations of rationalisation.

\begin{theorem}\label{th:rational solution}
There is an integer $i\geq 0$ such that $\mathbb{S}=\mathbb{R}^j(S^{\Delta})$ for each integer $j\geq i$.
\end{theorem}

This theorem shows the rational solution is well-defined.
The intuition behind it is that, with the initial solution $S^{\Delta}$, the iterative rationalisation process eliminates more and more but not all strategies (Lemma~\ref{lm:chain} in Appendix \ref{sec:app rational solution}). 
Since the strategy space is finite (item~\ref{dfitem:game strategy space} of Definition~\ref{df:game}), the elimination process has to reach a stable solution in finite steps and stays there forever (Lemma~\ref{lm:rational solution} in Appendix \ref{sec:app rational solution}). Such a stable solution is the rational solution by Definition~\ref{df:rational solution}.

In a sense, rational solution is a solution concept in uncommon RBR systems.
It predicts the strategic behaviours of all agents, both real and doxastic ones, in a game.
Indeed, what matters are the real agents.
The next definition extracts the elements of the real agents in a rational solution.

\begin{definition}\label{df:doxastic rationalisability}
The \textbf{doxastic rationalisability} of the game $(\Delta,\preceq)$ on the RBR graph $(N,E,\ell,\pi)$ is a family of sets $\mathfrak{R}=\{\mathfrak{R}_a\}_{a\in\mathcal{A}}$ such that
\begin{equation}\notag
\mathfrak{R}_a:=\begin{cases}
    \mathbb{S}_{\pi_a}, &\text{if $\pi_a$ is defined;}\\
    \Delta_a, &\text{otherwise;}
\end{cases}
\end{equation}
where $\mathbb{S}$ is the rational solution of the same game on the same RBR graph.
\end{definition}

Doxastic rationalisability, the proposed solution concept, is the exact extension of {\em rationalisability} \cite{pearce1984rationalizable,bernheim1984rationalizable} into uncommon RBR systems.
In other words, without any other assumption than RBR among the agents, doxastic rationalisability is the unique reasonable prediction\footnote{In the sense that (1) every strategy not in the doxastic rationalisability is believed to be dominated by another strategy, so no agent would like to choose it; (2) more assumptions/beliefs are required to eliminate a strategy in the doxastic rationalisability.} of the agents' strategic behaviours in a game.

\section{Equivalence in RBR Graphs}\label{sec:equivalence RBR graphs}

Intuitively, if there is no other assumption than uncommon RBR, then an agent's strategic behaviour is only affected by her own belief.
In this sense, if a (real or doxastic) agent always has the same strategic behaviour in two RBR systems, then we say that the agent has {\em equivalent} beliefs in these RBR systems.
Note that, a (real or doxastic) agent is denoted by a node in RBR graphs. 
For simplicity, we say that two nodes are {\em doxastically equivalent} if the agents denoted by them have the same strategic behaviour in every game.

Formally, we consider two nodes $n$ and $n'$ in two (possibly equal) arbitrary RBR graphs $B=(N,E,\ell,\pi)$ and $B'=(N',E',\ell',\pi')$, respectively.
For an arbitrary game $G=(\Delta,\preceq)$, denote by $\mathbb{S}(G)=\{\mathbb{S}(G)_m\}_{m\in N}$ and $\mathbb{S}'(G)=\{\mathbb{S}'(G)_{m'}\}_{m'\in N'}$ the rational solutions of the game $G$ on the RBR graphs $B$ and $B'$, respectively.
Then, doxastic equivalence between nodes is defined as follows.

\begin{definition}\label{df:doxastic equivalent nodes}
The nodes $n$ and $n'$ are \textbf{doxastically equivalent} if $\mathbb{S}(G)_{n}=\mathbb{S}'(G)_{n'}$ for each game $G$.
\end{definition}

Recall that, as discussed in Section~\ref{sec:introduction}, agent $a$ in Figure~\ref{fig:3Agent2CommonBelief} and in Figure~\ref{fig:3AgentUncommonBelief} have the same belief: either of them believes that there is RCBR with agent $b$ and agent $c$ is irrational.
The same goes with agent $b$ in Figure~\ref{fig:3Agent2CommonBelief} and in Figure~\ref{fig:3AgentUncommonBelief}, and agent $c$ in Figure~\ref{fig:3AgentCommonBelief} and in Figure~\ref{fig:3AgentUncommonBelief-doxastic}.
Note that, by ``the same belief'' we mean that the belief hierarchy (\textit{i.e.} the set of belief sequences) is the same.
In this sense, by saying that the agents denoted by nodes $n$ and $n'$ have the same belief, we mean $\Psi_n^*=\Psi_{n'}^*$.

We find that beliefs are equivalent if and only if they are the same, as formally stated in the following theorem.

\begin{theorem}\label{th:doxastic equivalent nodes = indistinguishable}
The nodes $n$ and $n'$ are doxastically equivalent if and only if $\Psi_n^*=\Psi_{n'}^*$.
\end{theorem}

For the ``only if'' part of Theorem~\ref{th:doxastic equivalent nodes = indistinguishable}, note that $\Psi_n^*\neq\Psi_{n'}^*$ implies the existence of an integer $k$ such that $\Psi_n^{k}\neq\Psi_{n'}^k$ by Definition~\ref{df:path set}.
Then, we show the existence of a parameterised game $G_k$ such that $\mathbb{S}(G_k)_{n}\neq\mathbb{S}'(G_k)_{n'}$ (Definition~\ref{df:sequence game} and Lemma~\ref{lm:sequence game rational solution} in Appendix \ref{sec:app doxastic equivalence necessary condition}).
For the ``if'' part of Theorem~\ref{th:doxastic equivalent nodes = indistinguishable}, we prove by induction that, for each integer $i\geq 1$, after the $i^{th}$ rationalisation, $\mathbb{R}_B^i(S^\Delta)_n=\mathbb{R}_{B'}^i(S^\Delta)_{n'}$ in every game (Lemma~\ref{lm:indistinguishable nodes are doxastic equivalent} in Appendix \ref{sec:app doxastic equivalence sufficient condition}).
Then, the ``if'' part statement of Theorem~\ref{th:doxastic equivalent nodes = indistinguishable} follows from Theorem~\ref{th:rational solution}.

Now, we consider the equivalence of RBR systems.
Recall that an RBR system is a collection of the belief hierarchies of all real agents. 
The real agents are whom we care about.
In this sense, we say that two RBR systems are equivalent if no real agents can distinguish them.
In other words, every real agent should always have the same strategic behaviour in two equivalent RBR systems.
Note that we use RBR graphs to denote RBR systems.
Formally, we consider the equivalence of two arbitrary RBR graphs $B$ and $B'$. For any game $G$, denote by $\mathfrak{R}(G)=\{\mathfrak{R}(G)_a\}_{a\in\mathcal{A}}$ and $\mathfrak{R}'(G)=\{\mathfrak{R}'(G)_a\}_{a\in\mathcal{A}}$ the doxastic rationalisabilities of the game $G$ on the RBR graphs $B$ and $B'$, respectively.

\begin{definition}\label{df:equivalent}
The RBR graphs $B$ and $B'$ are \textbf{equivalent} if $\mathfrak{R}(G)_a=\mathfrak{R}'(G)_a$ for each agent $a\in\mathcal{A}$ and each game $G$.
\end{definition}

The next theorem shows the necessary and sufficient condition for two RBR systems to be equivalent.
That is, for each agent $a$, either $a$ is irrational in both systems, or $a$ is rational and has equivalent beliefs in both systems. 
In the RBR graphs, the former means that agent $a$ is not in the domain of definition of the designating functions.
The latter is formally expressed with the doxastic equivalence between the nodes denoting agent $a$.
Formal proofs of the theorem and its corollary below can be found in Appendix~\ref{sec:app RBR graph equivalence}.

\begin{theorem}\label{th:RBR graph equivalent}
RBR graphs $(N,E,\ell,\pi)$ and $(N',E',\ell',\pi')$ are equivalent if and only if, for each agent $a\in\mathcal{A}$, \textbf{either} both $\pi_a$ and $\pi'_a$ are not defined, \textbf{or} $\pi_a$ and $\pi'_a$ are both defined and doxastically equivalent.
\end{theorem}

The next corollary follows directly from Theorem~\ref{th:RBR graph equivalent} and Theorem~\ref{th:doxastic equivalent nodes = indistinguishable}.
It shows that the core of the equivalence of two RBR systems is the belief hierarchy (\textit{i.e.} the set of all belief sequences, $\Psi_{\pi_a}^*$) of each real agent.
This property is used in the next section for minimising an RBR graph.

\begin{corollary}\label{cr:RBR graph equivalent}
RBR graphs $(N,E,\ell,\pi)$ and $(N',E',\ell',\pi')$ are equivalent if and only if $\pi$ and $\pi'$ have the same domain $\mathcal{D}$ of definition, \textbf{and} $\Psi_{\pi_a}^*=\Psi_{\pi'_a}^*$ for each agent $a\in\mathcal{D}$.
\end{corollary}

\section{Minimisation of RBR Graphs}\label{sec:minimisation RBR graph}

So far we have assumed that the RBR systems/graphs are given.
However, in most situations, this is not the case.
In behaviour economics, researchers study how to {\em elicit} the belief of a single agent \cite{schotter2014belief,charness2021experimental,danz2022belief}.
In application scenarios, it is probably the same: we elicit the belief hierarchy of each agent in a system and combine their beliefs as a whole.
For instance, to get the RBR system depicted in Figure~\ref{fig:3AgentUncommonBelief}, we first know that (1) agent $a$ believes RCBR exists between herself and agent $b$; (2) agent $b$ believes RCBR exists between herself and agent $a$; (3) agent $c$ believes that RCBR exists agents $a$ and $b$.
Then, we depict each agent's belief hierarchy with a graph and combine all of them as a whole, as shown in Figure~\ref{fig:3AgentUncommonBelief-collection}.
In other words, an RBR system is a collection of the belief hierarchy of each (real) agent in it.
Recall that, in Definition~\ref{df:RBR graph}, we never require an RBR graph to be a connected graph.

\begin{figure}[bht]
\begin{center}
\scalebox{0.5}{\includegraphics{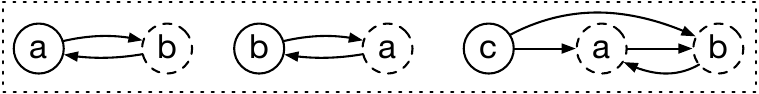}}
\caption{The collection of individual beliefs in Figure~\ref{fig:3AgentUncommonBelief}.
}
\label{fig:3AgentUncommonBelief-collection}
\end{center}
\end{figure}

It is easily observable and verifiable using Corollary~\ref{cr:RBR graph equivalent} that the RBR systems denoted in Figure~\ref{fig:3AgentUncommonBelief} and Figure~\ref{fig:3AgentUncommonBelief-collection} are indeed equivalent.
However, the RBR graph in Figure~\ref{fig:3AgentUncommonBelief-collection} has more than twice the number of nodes than that in Figure~\ref{fig:3AgentUncommonBelief}.
Note that, by Definition~\ref{df:rationalisation on solution}, the time complexity of rationalisation is proportional to the number of nodes in an RBR graph.
Minimising an RBR graph is to find an equivalent RBR graph with the fewest nodes.
On the one hand, it reduces the time complexity of computing the doxastic rationalisability.
On the other hand, it helps to find the most condensed expression of an RBR system.
Formally, we have the next definition for ``the most condensed expression''.

\begin{definition}\label{df:canonical RBR graph}
An RBR graph $(N,E,\ell,\pi)$ is \textbf{canonical} if $\Psi^*_{n}\neq\Psi^*_{n'}$ for all distinct nodes $n,n'\in N$.
\end{definition}

Intuitively, $\Psi^*_{n}$ denotes the belief hierarchy of the agent denoted by node $n$.
Then, an RBR graph is canonical if the nodes represent agents with different belief hierarchies.
It is proved in Appendix~\ref{sec:app canonical RBR graph} that {\em a canonical RBR graph is not equivalent to any RBR graph with fewer nodes} (Lemma~\ref{lm:canonical is minimal}) and {\em two equivalent canonical RBR graphs must be isomorphic} (Definition~\ref{df:RBR graph isomorphic} and Theorem~\ref{th:equivalent canonical graph are isomorphic}).
Due to the above two properties, to minimise an RBR graph, we only need to compute an equivalent canonical RBR graph.

Technically, $\Psi^*_{n}$ is the set of ``labelling sequences'' of all paths starting at node $n$. 
This reveals a similarity between RBR graphs and automata.
Inspired by Myhill–Nerode theorem \cite{myhill1957finite,nerode1958linear} and Hopcroft's algorithm \cite{hopcroft1971n}, we design Algorithm~\ref{alg:compute CF} that works based on {\em partition refinement} and outputs an equivalent canonical RBR graph of the input RBR graph, which is unique up to isomorphism.
A detailed explanation and formal proof of its correctness can be found in Appendices~\ref{sec:app partition sequence} and \ref{sec:app minimisation algorithm}.
The time complexity of Algorithm~\ref{alg:compute CF} is $O(|\mathcal{A}|\!\cdot\!|N|^2\!\cdot\!log|N|)$, where $|\mathcal{A}|$ is the number of agents and $|N|$ is the numbers of nodes in the input RBR graph.
See Appendix~\ref{sec:app complexity} for the complexity analysis.

\SetKwComment{Comment}{//}{}
\begin{algorithm}[hbt]
\footnotesize
\caption{Minimise an RBR graph}
\label{alg:compute CF}
\KwIn{RBR graph $(N,E,\ell,\pi)$}
\KwOut{RBR graph $(N', E',\ell',\pi')$}

$\mathbb{P}\leftarrow \big\{\{n'\in N\,|\,\ell_{n'}=\ell_n\}\,\big|\,n\in N\big\}$\label{algline:initialise P}{\scriptsize\Comment*[r]{$\Psi_n^1$ equivalence}\small}
$stable\leftarrow false$\label{algline:initialise flag}\;
\While({\scriptsize\Comment*[f]{$\Psi_n^i$ equiv. $\to\Psi_n^{i+1}$ equiv.}\small}){not $stable$\label{algline:while start}}{
    $stable\leftarrow true$\;
    $\mathbb{P}'\leftarrow\varnothing$\label{algline:initialise P'}\;
    \For{each set $P\in\mathbb{P}$\label{algline:partition for loop start}}{
    \For{each node $n\in P$\label{algline:type for loop start}}{
        $type(n)\leftarrow\{P'\in\mathbb{P}\,|\,nEn',n'\in P'\}$\label{algline:computing type}\;
    }
    $\mathbb{Q}\!\leftarrow\!\big\{\{n'\!\in\! P\mid type(n')\!=\!type(n)\}\mid n\!\in\! P\big\}$\label{algline:partition each P}\;
    $\mathbb{P}'\leftarrow\mathbb{P}'\cup\mathbb{Q}$\label{algline:update P'}\;
    \If{$|\mathbb{Q}|>1$\label{algline:flag update if}}{
        $stable\leftarrow false$\label{algline:partition for loop ends}\;
    }
    }
    $\mathbb{P}\leftarrow\mathbb{P}'$\label{algline:while end}\;
}
$N'\leftarrow\mathbb{P}$\label{algline:output node set}{\scriptsize\Comment*[r]{equivalent classes as nodes}\small}
$E'\leftarrow\{(P,Q)\mid (n,m)\in E,n\in P, m\in Q\}$\label{algline:output edge set}\;
\For{each node $P\in N'$\label{algline:output label for loop}}{
    pick an arbitrary node $n\in P$ and $\ell'_P\leftarrow\ell_{n}$ \label{algline:output label assign}\;
}
\For{each agent $a\in\mathcal{A}$\label{algline:output designation for loop}}{
    $\pi'_a\leftarrow P: \pi_a\in P$ if $\pi_a$ is defined\label{algline:output designation assign}\;
}
\Return{$(N', E',\ell',\pi')$}\label{algline:CF algorithm return}\;
\end{algorithm}

\section{Concluding Discussion}

Generally speaking, the RBR graph proposed in this paper is a {\em syntactic} presentation of an RBR system.
It is in line with our linguistic intuition about beliefs (\textit{i.e.} the correspondence between the label sequence of a path and a belief sequence in the hierarchy).
In a sense, doxastic rationalisability, the solution concept we propose, is a {\em semantic} interpretation of an RBR system in games.
From this perspective, Theorem~\ref{th:doxastic equivalent nodes = indistinguishable} and Corollary~\ref{cr:RBR graph equivalent} reveal the correlation between syntax and semantics of our graph-based language, based on which, we design an efficient algorithm that computes the most condensed syntactic expression of an RBR system.


\clearpage

\section*{Acknowledgements}
The research is funded by the China Scholarship Council (CSC No.202206070014).

\bibliography{this}

\clearpage

\appendix

\begin{center}
{\LARGE \bf Technical Appendix\footnote{This appendix is NOT a part of the AAAI-25 proceedings.}}

\end{center}
\vspace{2mm}


To better present the formal proofs and discussions, we start this appendix with a set of additional notations, definitions, and properties used in the other parts of the appendix.

\section{Additional Notations and Properties}

In addition to Definition~\ref{df:reasoning scene}, for two reasoning scenes $\Theta_a$ and $\bar{\Theta}_a$ of agent $a$ in one game, we say that $\Theta_a$ is a \textit{sup-scene} of $\bar{\Theta}_a$ and $\bar{\Theta}_a$ is a \textit{sub-scene} of $\Theta_a$ if $\Theta_a\sqsupseteq\bar{\Theta}_a$, which means $\Theta_a^b\supseteq \bar{\Theta}_a^{b}$ for each agent $b\neq a$.
Then, the lemma below follows from Definition~\ref{df:dominated strategy}.
\begin{lemma}\label{lm:dominance scale}
If $s_a\lhd_{\Theta_a}s'_a$ and $\Theta_a\sqsupseteq \Theta'_a$, then $s_a\lhd_{\Theta'_a}s'_a$.
\end{lemma}

The next lemma shows the monotonicity of rationalisation (Definition~\ref{df:rational response}) on the ``$\sqsupseteq$'' order of two reasoning scenes.
Intuitively, it says that an agent is more certain about which strategies are rationalisable if she is more certain about the opponents' strategic behaviours. 

\begin{lemma}\label{lm:rational response monotonicity}
If $\Theta_a\sqsupseteq \Theta'_a$, then 
$\Re_a(\Theta_a)\supseteq\Re_a(\Theta'_a)$.
\end{lemma}
\begin{proof}
Consider any strategy $s_a\!\in\!\Delta_a$ such that $s_a\!\notin\!\Re_a(\Theta_a)$. Then, there is a strategy $s'_a\in\Delta_a$ such that $s_a\lhd_{\Theta_a}s'_a$ by Definition~\ref{df:rational response}. Thus, $s_a\lhd_{\Theta'_a}s'_a$ by the assumption $\Theta_a\sqsupseteq \Theta'_a$ and Lemma~\ref{lm:dominance scale}. Hence, $s_a\notin\Re_a(\Theta'_a)$ by Definition~\ref{df:rational response}. 
\end{proof}

Next, we define two notations that capture the local properties of RBR graphs.
For an arbitrary node $n\in N$ in an RBR graph $(N,E,\ell,\pi)$, let
\begin{equation}\label{eq:1 step neighbour set}
Adj_n:=\{m\in N\,|\,nEm\}
\end{equation}
be the set of all adjacent nodes from node $n$ and 
\begin{equation}\label{eq:1 step agent set}
\mathcal{B}_n:=\{\ell_m\mid m\in Adj_n\}
\end{equation}
be the set of labels of the adjacent nodes of $n$.
Informally, in the RBR system denoted by the RBR graph $(N,E,\ell,\pi)$, the (real or doxastic) agent $\ell_n$ denoted by node $n$ believes that the agents in set $\mathcal{B}_n$ are rational and the others are irrational. The agents who are rational in agent $\ell_n$'s belief are denoted by the nodes in set $Adj_n$ in the RBR graph.

Note that $\ell_{m_1}\neq\ell_{m_2}$ for distinct nodes $m_1,m_2\in Adj_n$ by statement~\eqref{eq:1 step neighbour set} and item~\ref{dfitem:different child label} of Definition~\ref{df:RBR graph}. 
Thus, by statement~\eqref{eq:1 step agent set}, function $\ell$ forms a {\em bijection} from the set $Adj_n$ to the set $\mathcal{B}_n$. 
For this reason, we have the next definition.
\begin{definition}\label{df:function tau}
For any node $n$ in RBR graph $(N,E,\ell,\pi)$, the bijection $\beta_n:\mathcal{B}_n\to Adj_n$ is such that $\beta_n(b)$ is the node $m\in Adj_n$ where $\ell_m=b$.
\end{definition}

In other words, $\beta_n$ is the reversed function of $\ell$ in the adjacent area (\textit{i.e. $Adj_n$}) of node $n$.
Note that, by statement~\eqref{eq:1 step neighbour set}, statement~\eqref{eq:path set main} can be reformulated as follows:

\begin{equation}\label{eq:path set}
\Pi_n^i=
\begin{cases}
\{\ell_n\}, &\text{if } i=1;\\
\{\ell_n\!::\!\sigma\mid m\in Adj_n,\sigma\in \Pi_{m}^{i-1}\}, &\text{if }  i\geq 2.
\end{cases}
\end{equation}

Moreover, by statements~\eqref{eq:1 step neighbour set} and \eqref{eq:1 step agent set}, statement~\eqref{eq:RBR graph reasoning scene 1} can be reformulated as follows:
\begin{equation}\label{eq:RBR graph reasoning scene}
\Tilde{\Theta}_{n}^b(S)=
\begin{cases}
S_{\beta_n(b)}, & \text{if } b\in\mathcal{B}_n; \\
\Delta_b, & \text{otherwise}.
\end{cases}
\end{equation}

Consider an arbitrary nonempty sequence $\sigma$ of agents. If $\sigma=a\!::\!\rho$, then by $hd(\sigma)$ we mean agent $a$ and by $tl(\sigma)$ we mean (possibly empty) sequence $\rho$.
Then, the next three lemmas follow from statement~\eqref{eq:path set}.

\begin{lemma}\label{lm:March-1-b}
$hd(\sigma)=\ell_n$ for each sequence $\sigma\in\Pi_n^i$.
\end{lemma}

\begin{lemma}\label{lm:March-1-a}
$|\sigma|=i$ for each sequence $\sigma\in\Pi_n^i$.
\end{lemma}

\begin{lemma}\label{lm:March-1-c}
For any sequence $\sigma\in\Pi_n^i$, there is a path $p$ starting at node $n$ in the RBR graph such that $\sigma$ is the belief sequence of path $p$.
\end{lemma}

Note that, by item~\ref{dfitem:RBR graph labelling function 1} of Definition~\ref{df:RBR graph}, the belief sequence of any path in an RBR graph is an alternating sequence. 
Then, the next lemma follows from Lemma~\ref{lm:March-1-c}.

\begin{lemma}\label{lm:May-29-a}
Each sequence $\sigma\in\Pi_n^i$ is an alternating sequence.
\end{lemma}

The next two lemmas follow from statements~\eqref{eq:accumulated path set}, \eqref{eq:whole path set} and Lemma~\ref{lm:March-1-a}.
\begin{lemma}\label{lm:equal path set}
$\Psi_{n}^j=\Psi_{n'}^j$ if and only if $\Pi_{n}^i=\Pi_{n'}^i$ for each integer $i$ such that $0<i\leq j$.
\end{lemma}
\begin{lemma}\label{lm:equal whole path set}
$\Psi_{n}^*=\Psi_{n'}^*$ if and only if $\Pi_{n}^i=\Pi_{n'}^i$ for each integer $i\geq 1$.
\end{lemma}


\section{Rational Solution and Theorem~\ref{th:rational solution}}\label{sec:app rational solution}

We first consider two arbitrary solutions $S=\{S_n\}_{n\in N}$ and $S'=\{S'_n\}_{n\in N}$ of the same game on the same RBR graph. 
If $S_n\supseteq S'_n$ for each node $n\in N$, then we say that $S$ is a \textit{sup-solution} of $S'$ and $S'$ is a \textit{sub-solution} of $S$ and write $S\sqsupseteq S'$ and $S'\sqsubseteq S$, respectively.
Specifically, we write $S\sqsupsetneq S'$ if $S\sqsupseteq S'$ and there is a node $n$ such that $S_n\supsetneq S'_n$.
The \textit{union} of $S$ and $S'$ is the solution $\{S_n\cup S'_n\}_{n\in N}$ of the same game on the same RBR graph.
Then, the next lemma follows from Definition~\ref{df:belief scene} and statement~\eqref{eq:RBR graph reasoning scene}.
\begin{lemma}\label{lm:monotonicity belief scene}
For two solutions $S$ and $S'$ of the same game on the same RBR graph, if $S\sqsupseteq S'$, then the belief scenes $\tilde{\Theta}_n(S)\sqsupseteq\tilde{\Theta}_n(S')$ for each node $n$ in the RBR graph.
\end{lemma}

Recall that $\tilde{\Theta}_n(S)$ is belief scene of node $n$ on solution $S$ by Definition~\ref{df:belief scene}.
Then, by Lemma~\ref{lm:monotonicity belief scene}, if $S$ is a sup-solution of $S'$, then for each node $n$ in the RBR graph, the belief scene of $n$ on solution $S$ is a sup-scene of that on solution $S'$. 
Based on this property, we have the next lemma. It shows a kind of ``monotonicity'' in rationalisation on solutions.

\begin{lemma}\label{lm:rationalisation monotonicity}
For any solutions $S$ and $S'$ of the same game on the same RBR graph, if $S\sqsupseteq S'$, then $\mathbb{R}(S)\sqsupseteq\mathbb{R}(S')$.
\end{lemma}
\begin{proof}
Note that, $\Tilde{\Theta}_n(S)\sqsupseteq\Tilde{\Theta}_n(S')$ for any node $n$ in the RBR graph by Lemma~\ref{lm:monotonicity belief scene} and the assumption $S\sqsupseteq S'$. 
Then, $\Re_{\ell_n}(\Tilde{\Theta}_n(S))\supseteq\Re_{\ell_n}(\Tilde{\Theta}_n(S'))$ by Lemma~\ref{lm:rational response monotonicity}.
Therefore, $\mathbb{R}(S)\sqsupseteq\mathbb{R}(S')$ by Definition~\ref{df:rationalisation on solution}.
\end{proof}

The next lemma follows from Lemma~\ref{lm:rationalisation monotonicity} and Definition~\ref{df:ith rationalisation} by mathematical induction on integer $i$. 
\begin{lemma}\label{lm:i order rationalisation monotonicity}
For any solutions $S$ and $S'$ of the same game on the same RBR graph, if $S\sqsupseteq S'$, then $\mathbb{R}^i(S)\sqsupseteq\mathbb{R}^i(S')$ for each integer $i\geq 0$.
\end{lemma}

Next, we consider a property of stable solutions.
The next lemma follows from Definition~\ref{df:ith rationalisation} and Definition~\ref{df:stable solution}.
It shows that the iterative rationalisation process on a stable solution always stays in the same stable solution.
\begin{lemma}\label{lm:stable sequence}
$\mathbb{R}^{i}(S)=S$ for each stable solution $S$ and each integer $i\geq 0$. 
\end{lemma}

Note that, by Definition~\ref{df:rational solution}, the rational solution is the limit of the iterative rationalisation on (the whole-strategy-space) solution $S^\Delta$. 
Next, we show that the rational solution is well-defined (\textit{i.e.} the limit must exist).

\begin{lemma}\label{lm:chain}
$\mathbb{R}^0(S^{\Delta})\sqsupseteq\mathbb{R}^1(S^{\Delta})
\sqsupseteq\mathbb{R}^2(S^{\Delta})\sqsupseteq\dots$.
\end{lemma}
\begin{proof}
Note that $S^{\Delta}\sqsupseteq\mathbb{R}(S^{\Delta})$ because $\mathbb{R}(S^{\Delta})$ is a solution. 
Then, $\mathbb{R}^i(S^{\Delta})\sqsupseteq\mathbb{R}^{i+1}(S^{\Delta})$ for each $i\geq 0$ by Definition~\ref{df:ith rationalisation} and Lemma~\ref{lm:i order rationalisation monotonicity}.
\end{proof}

The above lemma shows that the iterative rationalisation process based on solution $S^{\Delta}$ produces a ``non-expanding'' sequence of solutions. 
The next lemma shows that the iterative rationalisation process reaches a stable solution after a finite number of iterations and stays in this stable solution forever.
The intuition behind this is that, in the ``non-expanding'' sequence, solutions cannot ``shrink'' forever because of the finiteness property.
Once ``shrinking'' stops, a stable solution is reached and the iterative rationalisation process is ``trapped'' there.

\begin{lemma}\label{lm:rational solution}
There exists an integer $i\geq 0$ such that $\mathbb{R}^j(S^{\Delta})$ is equal to $\mathbb{R}^i(S^{\Delta})$ for each integer $j\geq i$.
\end{lemma}
\begin{proof}
Note that $\mathbb{R}^{0}(S^{\Delta})_n=S_n^{\Delta}=\Delta_{\ell_n}$ for each node $n$ in an RBR graph by Definition~\ref{df:ith rationalisation} and statement~\eqref{eq:solution Delta}.
Then, $\mathbb{R}^{0}(S^{\Delta})_n$ is a finite set of strategies for each node $n$ by item~\ref{dfitem:game strategy space} of Definition~\ref{df:game}.
Thus, by Lemma~\ref{lm:chain} and the fact that the number of nodes is also finite, there must be an integer $i\geq 0$ such that $\mathbb{R}^{i}(S^{\Delta})=\mathbb{R}^{i+1}(S^{\Delta})$.
This means $\mathbb{R}^{i}(S^{\Delta})$ is a stable solution by  Definition~\ref{df:stable solution}.
Hence, the statement of this lemma follows from Lemma~\ref{lm:stable sequence}.
\end{proof}

Lemma~\ref{lm:rational solution} shows the existence of an integer $i$ such that the sequence of solutions in Lemma~\ref{lm:chain} stabilises after the $i^{th}$ element. By Definition~\ref{df:rational solution}, the rational solution is equal to this stable element. Hence, we can conclude Theorem \ref{th:rational solution} in the main text.

\noindent\textbf{Theorem \ref{th:rational solution}} \textit{There is an integer $i\geq 0$ such that $\mathbb{S}=\mathbb{R}^j(S^{\Delta})$ for each integer $j\geq i$.}

\section{Doxastic Equivalence Between Nodes}\label{sec:app doxastic equivalence}

Theorem \ref{th:doxastic equivalent nodes = indistinguishable} in the main text states the necessary and sufficient condition for the doxastic equivalence of two nodes.
We formally prove the necessity in Subsection~\ref{sec:app doxastic equivalence necessary condition} and the sufficiency in Subsection~\ref{sec:app doxastic equivalence sufficient condition}.

\subsection{Necessity}\label{sec:app doxastic equivalence necessary condition}

The necessity of the condition $\Psi_n^*=\Psi_{n'}^*$ for doxastic equivalence of nodes $n$ and $n'$ is formally stated as the ``only if'' part of Theorem~\ref{th:doxastic equivalent nodes = indistinguishable}.
In this subsection, for two nodes $n$ and $n'$ such that $\Psi_n^*\neq\Psi_{n'}^*$, we construct a game $G_k$ that distinguishes them (\textit{i.e.} $\mathbb{S}(G_k)_n\neq\mathbb{S}'(G_k)_{n'}$).

We first introduce some notations used in the definition of game $G_k$.
For each agent $a\in\mathcal{A}$, let $\Sigma^i_a$ be the set of all \textit{nonempty alternating} sequences of agents of length at most $i$ and starting with agent $a$. Formally,
\begin{equation}\label{eq:Sigma^k_a}
\hspace{-2mm}
\Sigma^i_a:=
\begin{cases}
\varnothing, &\text{if }i=0;\\
\{a\}\cup\!\!\bigcup\limits_{b\in\mathcal{A}\setminus\{a\}}\!\!\big\{a\!::\!\sigma\mid\sigma\in\Sigma_b^{i-1}\big\}, &\text{if }i\geq 1.
\end{cases}
\end{equation}

The next lemma is true because any sequence of length at most $i-1$ is also a sequence of length at most $i$. It can be formally proved by induction on $i$ using statement~\eqref{eq:Sigma^k_a}.
\begin{lemma}\label{lm:Mar-5-1}
$\Sigma_a^{i-1}\subseteq\Sigma_a^{i}$ for any agent $a\in\mathcal{A}$ and any integer $i\geq 1$.
\end{lemma}

Recall that, by Lemma~\ref{lm:March-1-b}, Lemma~\ref{lm:March-1-a}, Lemma~\ref{lm:May-29-a} and statement~\eqref{eq:accumulated path set}, each sequence $\sigma\in\Psi_n^i$ is an alternating sequence of length at most $i$ and starting with agent $\ell_n$.
Then, the next two lemmas are straightforward.

\begin{lemma}\label{lm:Mar-5-2}
$\Psi_n^i\subseteq\Sigma_{\ell_n}^i$ for any integer $i\geq 0$ and any node $n$ in an RBR graph.
\end{lemma}

\begin{lemma}\label{lm:Mar-5-3}
$\Psi_{n}^i\cap\Sigma_{\ell_n}^j=\Psi_{n}^{\min(i,j)}$ for any node $n$ in an RBR graph and any integers $i,j\geq 0$.
\end{lemma}

Using the notations above, we define a game $G_k$ parameterised by integer $k$ as follows.

\begin{definition}\label{df:sequence game}
Game $G_k$ where $k\geq 1$ is a tuple $(\Delta,\preceq)$ such that, for each agent $a\in\mathcal{A}$,
\begin{enumerate}
\item\label{dfitem:sequence game strategy space} $\Delta_a=\Sigma^k_a\cup\{\bot_a\}$;
\item\label{dfitem:sequence game preference} $\s\preceq_a\s'$ if and only if $u_a(\s)\leq u_a(\s')$ for all outcomes $\s,\s'\in\prod_{b\in\mathcal{A}}\Delta_b$, where
\begin{equation} \label{eq:sequence game utility function}
u_a(\s):=
\begin{cases}
0, & \text{if } s_a=\bot_a;\\
1, & \text{if } s_a\in\Sigma^k_a \text{ and } \exists b\in\mathcal{A} (tl(s_a)=s_b);\\
-1, & \text{otherwise}.
\end{cases}
\end{equation}
\end{enumerate}
\end{definition}

Informally, $G_k$ is a game among the agents in set $\mathcal{A}$.
In this game, each agent $a\in\mathcal{A}$ can either choose an alternating sequence starting with $a$ and of length at most $k$ (\textit{i.e.} $s_a\in\Sigma_a^k$) or choose to quit the competition (\textit{i.e.} $s_a=\bot_a$).
A collection of all agents' choices forms an outcome. 
Each agent $a$'s preference over all possible outcomes is defined based on a utility function $u_a$ which maps each outcome to a value in the set $\{-1,0,1\}$.
Intuitively, if agent $a$ quits the competition (\textit{i.e.} $s_a=\bot_a$), then she neither wins nor loses, which is captured by $u_a(\s)=0$.
If agent $a$ does not quit the competition, then she must choose a sequence from set $\Sigma_a^k$.
We say that agent $a$ \textit{overrides} agent $b$ in outcome $\s$ if $tl(s_a)=s_b$.
Then, an agent $a$ who does not quit the competition wins (captured by $u_a(\s)=1$) if she overrides another agent and loses (captured by $u_a(\s)=-1$) otherwise.

Note that the empty sequence is not in set $\Sigma_a^k$ for each agent $a\in\mathcal{A}$ and each integer $k\geq 0$.
Then, agent $a$ loses for sure by choosing the singleton sequence $(a)\in\Sigma_a^k$ since $tl(a)$ is an empty sequence.
In other words, the singleton sequences are not rationalisable.
Thus, {\em no rational agent would choose a singleton sequence}.

Consider the situation when agent $a$ chooses a sequence $s_a\in\Sigma_a^k$.
Observe that, for each agent $b\in\mathcal{A}$, every sequence in set $\Sigma_b^k$ starts with $b$ by statement~\eqref{eq:Sigma^k_a}.
Then, $tl(s_a)\in\Delta_b$ if and only if $b=hd(tl(s_a))$ by item~\ref{dfitem:sequence game strategy space} of Definition~\ref{df:sequence game}.
Let $b$ be the agent $hd(tl(s_a))$.
Then, by choosing sequence $s_a$, agent $a$ cannot override any agent other than $b$.
Meanwhile, if we consider a specific reasoning scene $\Theta_a$ of agent $a$, then in agent $a$'s mind, agent $b$ chooses from set $\Theta_a^b$.
Thus, agent $a$ cannot override agent $b$ if $tl(s_a)\notin\Theta_a^b$.
Hence, {\em any strategy $s_a\in\Sigma_a^k$ such that $tl(s_a)\notin\Theta_a^b$ is not rationalisable}.

The next lemma formally captures the above two observations.
It has a stronger statement that a competing strategy $s_a\in\Sigma_a^k$ is not rationalisable only in the above two situations.
It also considers the situation when an agent decides to quit the competition (\textit{i.e. $s_a=\bot_a$}).

\begin{lemma}\label{lm:sequence game best response}
For any reasoning scene $\Theta_a$ of agent $a$ in game $G_k$, if $\bot_b\in \Theta^b_a$ for each agent $b\in\mathcal{A}\setminus\{a\}$, then 
\begin{equation}\notag
\Delta_a\setminus\Re_a(\Theta_a)=\{a\}\cup\!\!\bigcup_{b\in\mathcal{A}\setminus\{a\}}\!\!\big\{a\!::\!\sigma\mid \sigma\in\Sigma^{k-1}_b\setminus\Theta_a^b\big\}.
\end{equation}
\end{lemma}
\begin{proof}
We use $LHS$ and $RHS$ to denote the left-hand side and the right-hand side of the above formula, respectively.

($LHS\subseteq RHS$): Consider any strategy $s_a\notin RHS$. By statement~\eqref{eq:Sigma^k_a} and item~\ref{dfitem:sequence game strategy space} of Definition~\ref{df:sequence game}, either $s_a=\bot_a$ or $tl(s_a)\in\Theta_a^b$ for some agent $b\in\mathcal{A}\setminus\{a\}$.

If $s_a=\bot_a$, then consider an outcome $\s$ such that $s_b=\bot_b$ for each agent $b\in\mathcal{A}$. 
Note that $u_a(\s)=0$ and $u_a(\s_{-a},s'_a)=-1$ for any strategy $s'_a\neq \bot_a$ by statement~\eqref{eq:sequence game utility function}. 
Then, $\s\nprec_a (\s_{-a},s'_a)$ for any strategy $s'_a\in\Delta_a$ by item~\ref{dfitem:sequence game preference} of  Definition~\ref{df:sequence game}. 
Thus, $\bot_a\nlhd_{\Theta_a}s'_a$ for any strategy $s'_a\in\Delta_a$ by Definition~\ref{df:dominated strategy}.
Hence, $\bot_a\in \Re_a(\Theta_a)$ by Definition~\ref{df:rational response}.
Therefore, $\bot_a\notin LHS$.

If $tl(s_a)\in\Theta_a^b$ for an agent $b\in\mathcal{A}\setminus\{a\}$, then $tl(s_a)\in\Sigma^k_b$ by item~\ref{dfitem:sequence game strategy space} of Definition~\ref{df:sequence game} because $tl(s_a)\neq\bot_b$.
Consider any tuple $\s_{-a}\in\Theta_a$ such that $s_b=tl(s_a)$.
Then, $u_a(\s_{-a},s_a)=1$ by statement~\eqref{eq:sequence game utility function}. 
Note that $1$ is the highest utility in game $G_k$.
Thus, $(\s_{-a},s_a)\nprec_a (\s_{-a},s'_a)$ for any strategy $s'_a\in\Delta_a$ by item~\ref{dfitem:sequence game preference} of  Definition~\ref{df:sequence game}. 
Then, $s_a\nlhd_{\Theta_a}s'_a$ for any strategy $s'_a\in\Delta_a$ by Definition~\ref{df:dominated strategy}. 
Hence, $s_a\in \Re_a(\Theta_a)$ by Definition~\ref{df:rational response}.
Therefore, $s_a\notin LHS$.

($LHS\supseteq RHS$): Consider any strategy $s_a\in RHS$. Then, $tl(s_a)\notin\Theta_a^b$ for each agent $b\in\mathcal{A}\setminus\{a\}$.
Note that $u_a(\s_{-a},s_a)=-1$ and $u_a(\s_{-a},\bot_a)=0$ for each tuple $\s_{-a}\in\Theta_a$ by statement~\eqref{eq:sequence game utility function}. 
Then, $(\s_{-a},s_a)\prec_a (\s_{-a},\bot_a)$ for each tuple $\s_{-a}\in\Theta_a$ by item~\ref{dfitem:sequence game preference} of  Definition~\ref{df:sequence game}. 
Thus, $s_a\lhd_{\Theta_a}\bot_a$ by Definition~\ref{df:dominated strategy}. 
Hence, $s_a\notin \Re_a(\Theta_a)$ by Definition~\ref{df:rational response}.
Therefore, $s_a\in LHS$.
\end{proof}

The above lemma shows which strategies should be eliminated in a rationalisation.
Using this result, the next lemma considers the result of iterative rationalisation of game $G_k$ on an arbitrary RBR graph.
It shows that the $i^{th}$ rationalisation result in game $G_k$ is related to the belief hierarchy (\textit{i.e.} $\Psi_n^i$) bounded by integer $i$ for each (real or doxastic) agent denoted by a node in the RBR graph.

\begin{lemma}\label{lm:i rationalisation on sequence game}
For the game $G_k=(\Delta,\preceq)$, any integer $i$ such that $i\geq 0$, and any node $n$ in the RBR graph $(N,E,\ell,\pi)$,
\begin{equation}\notag
\mathbb{R}^i(S^\Delta)_n=\Delta_{\ell_n}\setminus\Psi_n^i.
\end{equation}
\end{lemma}
\begin{proof}
We prove the statement of the lemma by induction on integer $i$.
For the base case where $i=0$, by Definition~\ref{df:ith rationalisation}, statements~\eqref{eq:solution Delta} and \eqref{eq:accumulated path set},
\begin{equation}\notag
\mathbb{R}^0(S^\Delta)_n=\Delta_{\ell_n}=\Delta_{\ell_n}\setminus\varnothing=\Delta_{\ell_n}\setminus\Psi_n^0.
\end{equation}

Next, we consider the cases where $i\geq 1$.
By Definition~\ref{df:ith rationalisation} and Definition~\ref{df:rationalisation on solution},
\begin{equation}\label{eq:Mar-4-4}
\mathbb{R}^i(S^\Delta)_n=\mathbb{R}(\mathbb{R}^{i-1}(S^\Delta))_n=\Re_{\ell_n}(\tilde{\Theta}_n(\mathbb{R}^{i-1}(S^\Delta))).
\end{equation}
Note that, by statement~\eqref{eq:RBR graph reasoning scene}, for each agent $b\neq\ell_n$,
\begin{equation}\label{eq:Mar-4-3}
\tilde{\Theta}^b_n(\mathbb{R}^{i-1}(S^\Delta))=
\begin{cases}
\mathbb{R}^{i-1}(S^\Delta)_{\beta_n(b)}, & \text{if } b\in\mathcal{B}_n;\\
\Delta_b, & \text{otherwise}.
\end{cases}
\end{equation}
Besides, by the induction hypothesis, for each agent $b\in\mathcal{B}_n$,
\begin{equation}\notag
\mathbb{R}^{i-1}(S^\Delta)_{\beta_n(b)}=\Delta_b\setminus\Psi_{\beta_n(b)}^{i-1}.
\end{equation}
Thus, by statement~\eqref{eq:Mar-4-3}, for each agent $b\neq\ell_n$,
\begin{equation}\label{eq:Mar-4-5}
\tilde{\Theta}^b_n(\mathbb{R}^{i-1}(S^\Delta))=
\begin{cases}
\Delta_b\setminus\Psi_{\beta_n(b)}^{i-1}, & \text{if } b\in\mathcal{B}_n;\\
\Delta_b, & \text{otherwise}.
\end{cases}
\end{equation}
Observe that, by Lemma~\ref{lm:Mar-5-1} and item~\ref{dfitem:sequence game strategy space} of Definition~\ref{df:sequence game},
\begin{equation}\label{eq:3-9-2}
\Sigma^{k-1}_b\subseteq\Sigma^{k}_b\subseteq\Delta_b  
\end{equation}
for each agent $b\in\mathcal{A}$.
Then,
\begin{equation}\label{eq:3-10-1}
\Sigma^{k-1}_b\setminus\Delta_b=\varnothing.
\end{equation}
Also note that, for any sets $A$, $B$, and $C$, if $A\subseteq B$, then $A\setminus(B\setminus C)=A\cap C$.
Then, by statement~\eqref{eq:3-9-2},
\begin{equation}\label{eq:Mar-4-7}
\Sigma^{k-1}_b\setminus\left(\Delta_b\setminus\Psi_{\beta_n(b)}^{i-1}\right)=\Sigma^{k-1}_b\cap\Psi_{\beta_n(b)}^{i-1}.
\end{equation}
Let 
\begin{equation}\label{eq:Mar-4-8}
\hat{i}:=\min(i,k).
\end{equation}
Then, $\min(i\!-\!1,k\!-\!1)=\hat{i}-1$.
Thus, $\Psi_{\beta_n(b)}^{i-1}\cap\Sigma^{k-1}_b=\Psi_{\beta_n(b)}^{\hat{i}-1}$ by Lemma~\ref{lm:Mar-5-3} and the fact that $\ell_{\beta_n(b)}=b$ by Definition~\ref{df:function tau}.
Thus, by statement~\eqref{eq:Mar-4-7},
\begin{equation}\notag
\Sigma^{k-1}_b\setminus\left(\Delta_b\setminus\Psi_{\beta_n(b)}^{i-1}\right)=\Psi_{\beta_n(b)}^{\hat{i}-1}.
\end{equation}
Hence, by statements~\eqref{eq:Mar-4-5} and \eqref{eq:3-10-1}, for each agent $b\neq\ell_n$,
\begin{equation}\label{eq:3-9-1}
\Sigma^{k-1}_b\setminus\tilde{\Theta}^b_n(\mathbb{R}^{i-1}(S^\Delta))=
\begin{cases}
\Psi_{\beta_n(b)}^{\hat{i}-1}, & \text{if } b\in\mathcal{B}_n;\\
\varnothing, & \text{otherwise}.
\end{cases}
\end{equation}
Note that, by statement~\eqref{eq:Mar-4-4},
\begin{equation}\notag
\Delta_{\ell_n}\setminus\mathbb{R}^i(S^\Delta)_n=\Delta_{\ell_n}\setminus\Re_{\ell_n}(\tilde{\Theta}_n(\mathbb{R}^{i-1}(S^\Delta))).
\end{equation}
Then, by Lemma~\ref{lm:sequence game best response}.
\begin{equation}\notag
\begin{aligned}
&\Delta_{\ell_n}\setminus\mathbb{R}^i(S^\Delta)_n\\
=&\{\ell_n\}\cup\bigcup_{b\in\mathcal{A}\setminus\{a\}}\{\ell_n\!::\!\sigma\mid\sigma\in\Sigma^{k-1}_b\setminus\tilde{\Theta}^b_n(\mathbb{R}^{i-1}(S^\Delta))\big\}.
\end{aligned}
\end{equation}
Thus, by statement~\eqref{eq:3-9-1},
\begin{equation}\notag
\Delta_{\ell_n}\setminus\mathbb{R}^i(S^\Delta)_n=\{\ell_n\}\cup\bigcup_{b\in\mathcal{B}_n}\big\{\ell_n\!::\!\sigma\mid\sigma\in\Psi_{\beta_n(b)}^{\hat{i}-1}\big\}.
\end{equation}
Then, by Definition~\ref{df:function tau},
\begin{equation}\notag
\Delta_{\ell_n}\setminus\mathbb{R}^i(S^\Delta)_n=\{\ell_n\}\cup\!\!\!\!\bigcup_{m\in Adj_n}\!\!\!\!\big\{\ell_n\!::\!\sigma\mid\sigma\in\Psi_{m}^{\hat{i}-1}\big\}.
\end{equation}
Hence, by Definition~\ref{df:path set} and statement~\eqref{eq:path set},
\begin{equation}\label{eq:Mar-4-10}
\begin{aligned}
&\Delta_{\ell_n}\setminus\mathbb{R}^i(S^\Delta)_n\\
&=\Pi_n^1\cup\bigcup_{m\in Adj_n}\big\{\ell_n\!::\!\sigma\mid\sigma\in\bigcup_{0<j\leq\hat{i}-1}\Pi_{m}^{j}\big\}\\
&=\Pi_n^1\cup\bigcup_{0<j\leq \hat{i}-1}\bigcup_{m\in Adj_n}\big\{\ell_n\!::\!\sigma\mid\sigma\in\Pi_{m}^{j}\big\}\\
&=\Pi_n^1\cup\bigcup_{0<j\leq \hat{i}-1}\big\{\ell_n\!::\!\sigma\mid m\in Adj_n, \sigma\in\Pi_{m}^{j}\big\}\\
&=\Pi_n^1\cup\bigcup_{0<j\leq \hat{i}-1}\Pi_n^{j+1}
=\Psi_n^{\hat{i}}.
\end{aligned}
\end{equation}
Note that, by statement~\eqref{eq:Mar-4-8} and Lemma~\ref{lm:Mar-5-3},
\begin{equation}\label{eq:Mar-4-9}
\Sigma_{\ell_n}^k\setminus\Psi_n^{\hat{i}}
=\Sigma_{\ell_n}^k\setminus(\Sigma_{\ell_n}^k\cap\Psi_n^i)
=\Sigma_{\ell_n}^k\setminus\Psi_n^i.
\end{equation}
Meanwhile, $\bot_{\ell_n}\notin\Psi_n^{\hat{i}}$.
Therefore, 
\begin{equation}\notag
\begin{aligned}
&\mathbb{R}^i(S^\Delta)_n=\Delta_{\ell_n}\setminus\Psi_n^{\hat{i}}=(\{\bot_{\ell_n}\}\cup\Sigma_{\ell_n}^k)\setminus\Psi_n^{\hat{i}}\\
&\!\!=\!\{\bot_{\ell_n}\}\cup(\Sigma_{\ell_n}^k\setminus\Psi_n^{\hat{i}})=\{\bot_{\ell_n}\}\cup(\Sigma_{\ell_n}^k\setminus\Psi_n^i)=\Delta_{\ell_n}\!\setminus\!\Psi_n^i
\end{aligned}
\end{equation}
by statements~\eqref{eq:Mar-4-10}, \eqref{eq:Mar-4-9}, and item~\ref{dfitem:sequence game strategy space} of Definition~\ref{df:sequence game}.
\end{proof}

Recall that, as stated in Theorem~\ref{th:rational solution}, the rational solution can be reached by a finite number of iterative rationalisations.
Using the result in the above lemma, next, we consider the rational solution of game $G_k$ on an arbitrary RBR graph.

\begin{lemma}\label{lm:sequence game rational solution}
$\mathbb{S}(G_k)_n=\Delta_{\ell_n}\setminus\Psi_n^k$ for the game $G_k$ and any node $n$ in the RBR graph $(N,E,\ell,\pi)$.
\end{lemma}
\begin{proof}
Note that $\bot_{\ell_n}\notin\Psi_n^k$ by Definition~\ref{df:path set}.
Then, by item~\ref{dfitem:sequence game strategy space} of Definition~\ref{df:sequence game},
\begin{equation}\notag
\begin{aligned}
&\Delta_{\ell_n}\setminus\Psi_n^i
=(\{\bot_{\ell_n}\}\cup\Sigma_{\ell_n}^k)\setminus\Psi_n^i\\
&=\{\bot_{\ell_n}\}\cup(\Sigma_{\ell_n}^k\setminus\Psi_n^i)
=\{\bot_{\ell_n}\}\cup(\Sigma_{\ell_n}^k\setminus(\Sigma_{\ell_n}^k\cap\Psi_n^i)).
\end{aligned}
\end{equation}
Note that, $\Sigma_{\ell_n}^k\cap\Psi_n^i=\Psi_n^k$ when $i\geq k$ by Lemma~\ref{lm:Mar-5-3}.
Thus,
\begin{equation}\notag
\begin{aligned}
&\Delta_{\ell_n}\setminus\Psi_n^i=\{\bot_{\ell_n}\}\cup(\Sigma_{\ell_n}^k\setminus\Psi_n^k)\\
&=(\{\bot_{\ell_n}\}\cup\Sigma_{\ell_n}^k)\setminus\Psi_n^k=\Delta_{\ell_n}\setminus\Psi_n^k
\end{aligned}
\end{equation}
when $i\geq k$. 
Therefore, the statement of the lemma follows from Lemma~\ref{lm:i rationalisation on sequence game} and Definition~\ref{df:rational solution}.
\end{proof}

The above lemma shows the exact form of a rational solution of game $G_k$ on an RBR graph, which is closely related to the structural feature of the RBR graph. 
Using this result, we next prove the ``only if'' part of Theorem~\ref{th:doxastic equivalent nodes = indistinguishable}.

\noindent\textbf{Theorem \ref{th:doxastic equivalent nodes = indistinguishable}} (``only if'' part) \textit{The nodes $n,n'$ are \textbf{not} doxastically equivalent if $\Psi_n^*\neq \Psi_{n'}^*$.}

\begin{proof}
By Lemma~\ref{lm:equal whole path set} and the assumption $\Psi_n^*\neq \Psi_{n'}^*$, there is a minimal integer $i\geq 1$ such that
\begin{equation}\label{eq:4-4-1}
\Pi_n^i\neq \Pi_{n'}^i
\end{equation}
Then, for each integer $j$ such that $1\leq j<i$,
\begin{equation}\label{eq:4-4-2}
\Pi_n^j=\Pi_{n'}^j.
\end{equation}
Consider the game $G_i$ by setting parameter $k$ to $i$ in game $G_k$.
Then, by Lemma~\ref{lm:sequence game rational solution},
\begin{equation}\label{eq:Mar-4-11}
\mathbb{S}(G_i)_n=\Delta_{\ell_n}\setminus\Psi_n^i \text{\hspace{1mm} and \hspace{1mm}} \mathbb{S}'(G_i)_{n'}=\Delta_{\ell'_{n'}}\setminus\Psi_{n'}^i.
\end{equation}

If $i=1$, then $\ell_n\neq\ell'_{n'}$ by statements~\eqref{eq:path set main} and \eqref{eq:4-4-1}. 
Meanwhile, by statement~\eqref{eq:Mar-4-11}, item~\ref{dfitem:sequence game strategy space} of Definition~\ref{df:sequence game}, statement~\eqref{eq:Sigma^k_a}, and Definition~\ref{df:path set},
\begin{equation}\notag
\begin{aligned}
\mathbb{S}(G_i)_n=&\Delta_{\ell_n}\setminus\Psi^1_n=(\{\bot_{\ell_n}\}\cup\Sigma^1_{\ell_n})\setminus\Psi_{n}^1\\
&=(\{\bot_{\ell_n}\}\cup\{\ell_n\})\setminus\{\ell_n\}=\{\bot_{\ell_n}\}
\end{aligned}
\end{equation}
and similarly $\mathbb{S}'(G_i)_{n'}\!=\!\{\bot_{\ell'_{n'}}\}$.
Thus, $\mathbb{S}(G_i)_n\neq\mathbb{S}'(G_i)_{n'}$.
Hence, the statement of the theorem is true by Definition~\ref{df:doxastic equivalent nodes}.

If $i\geq 2$, then $\Pi_n^1=\Pi_{n'}^1$ by statement~\eqref{eq:4-4-2}.
This implies $\ell_n=\ell'_{n'}$ by statement~\eqref{eq:path set}.
Hence, by item~\ref{dfitem:sequence game strategy space} of Definition~\ref{df:sequence game},
\begin{equation}\label{eq:Mar-4-12}
\Delta_{\ell_n}=\Delta_{\ell'_{n'}}.
\end{equation}
Moreover, by Lemma~\ref{lm:Mar-5-2} and by item~\ref{dfitem:sequence game strategy space} of Definition~\ref{df:sequence game},
\begin{equation}\label{eq:Mar-5-1}
\Psi_n^i\subseteq\Sigma^i_{\ell_n}\subsetneq\Delta_{\ell_n} \text{\hspace{1mm} and \hspace{1mm}} \Psi_{n'}^i\subseteq\Sigma^i_{\ell'_{n'}}\subsetneq\Delta_{\ell'_{n'}}.
\end{equation}
However, $\Psi_n^i\neq \Psi_{n'}^i$ by statements~\eqref{eq:accumulated path set} and \eqref{eq:4-4-1}.
Then, $\mathbb{S}(G_i)_n\neq\mathbb{S}'(G_i)_{n'}$ by statements~\eqref{eq:Mar-4-11}, \eqref{eq:Mar-4-12}, and \eqref{eq:Mar-5-1}.
Thus, the statement of the theorem follows from Definition~\ref{df:doxastic equivalent nodes}.
\end{proof}

\subsection{Sufficiency}\label{sec:app doxastic equivalence sufficient condition}

The sufficiency of the condition $\Psi_n^*=\Psi_{n'}^*$ for doxastic equivalence of nodes $n$ and $n'$ is formally stated as the ``if'' part of Theorem~\ref{th:doxastic equivalent nodes = indistinguishable}.
In this subsection, for two nodes $n$ and $n'$ such that $\Psi_n^*=\Psi_{n'}^*$, we show that the iterative rationalisation process on them is a bisimulation and thus the rational solution on them is the same for each game.

To formally prove the sufficiency, we need a few more notations.
We still consider the nodes $n$ and $n'$ in the RBR graphs $B=(N,E,\ell,\pi)$ and $B'=(N',E',\ell',\pi')$, respectively.
For an arbitrary game $G=(\Delta,\preceq)$, denote by
$\mathbb{R}_B^i(S^\Delta)$ and $\mathbb{R}_{B'}^i(S^\Delta)$ the results of the $i^{th}$ rationalisation of the solution $S^\Delta$ on the RBR graphs $B$ and $B'$, respectively.
The next lemma presents a simple observation. We list it here for use in the proofs later.

\begin{lemma}\label{lm:3-11-a}
$\ell_n=\ell'_{n'}$ if $\Psi_{n}^{i}=\Psi_{n'}^{i}$ for some integer $i\geq 1$.
\end{lemma}
\begin{proof}
Note that, $\Pi_n^1=\Pi_{n'}^1$ is always true by Lemma~\ref{lm:equal path set} and the assumption $\Psi_n^{i}=\Psi_{n'}^{i}$ of the lemma.
Then, the statement of the lemma follows from statement~\eqref{eq:path set}.
\end{proof}

Recall that set $\Psi_n^*$ is a structural feature of the RBR graph concerning node $n$.
The next two lemmas consider the ``diffusion'' of such a structural feature via the edges in an RBR graph.
Intuitively, set $\Psi_n^i$ illustrates a partial belief of the agent denoted by node $n$.
Specifically, the assumption $\Psi_n^i=\Psi_{n'}^i$ captures a ``belief similarity'' between the agents denoted by nodes $n$ and $n'$.
Then, Lemma~\ref{lm:Feb-21-a} and Lemma~\ref{lm:4-4-a} below consider the ``belief similarity'' between the agents denoted by the nodes $m\in Adj_n$ and $m'\in Adj_{n'}$ that are labelled with the same agent.

\begin{lemma}\label{lm:Feb-21-a}
For any integer $i\geq 1$, any node $m\in Adj_n$, and any node $m'\in Adj_{n'}$, if $\Psi_{n}^{i}=\Psi_{n'}^{i}$  and $\ell_m=\ell'_{m'}$, then $\Psi_{m}^{i-1}=\Psi_{m'}^{i-1}$.
\end{lemma}
\begin{proof}
We prove this lemma by contradiction.
Note that, by Lemma~\ref{lm:3-11-a} and the assumption $\Psi_n^{i}=\Psi_{n'}^{i}$,
\begin{equation}\label{eq:3-11-1}
\ell_n=\ell'_{n'}.
\end{equation}
Suppose that $\Psi_{m}^{i-1}\neq\Psi_{m'}^{i-1}$.
More specifically, without loss of generality, suppose there is a sequence $\rho$ of agents such that $\rho\in\Psi_{m}^{i-1}$ and $\rho\notin\Psi_{m'}^{i-1}$.
Then, by statement~\eqref{eq:accumulated path set}, there is an integer $j$ such that $0<j\leq i-1$
and
\begin{equation}\label{eq:Mar-1-3}
\rho\in\Pi_m^j.
\end{equation}
Thus, $\ell_n\!::\!\rho\in\Pi_n^{j+1}$ by statement~\eqref{eq:path set} and the assumption $m\in Adj_n$ of the lemma.
Hence, by statement~\eqref{eq:accumulated path set} and the fact that $0<j\leq i-1$,
\begin{equation}\label{eq:Mar-1-9}
\ell_n\!::\!\rho\in\Psi_n^{i}.
\end{equation}

On the other hand, by the assumption $\rho\notin\Psi_{m'}^{i-1}$ and statement~\eqref{eq:accumulated path set}, for each integer $k$ such that $0<k\leq i-1$,
\begin{equation}\label{eq:Mar-1-6}
\rho\notin\Pi_{m'}^{k}.
\end{equation}
Moreover, by the assumption $\ell_m=\ell_{m'}$, statement~\eqref{eq:Mar-1-3}, and Lemma~\ref{lm:March-1-b},
\begin{equation}\label{eq:Mar-1-4}
hd(\rho)=\ell_m=\ell_{m'}.
\end{equation}
Meanwhile, $\ell_{m'}\neq\ell_{m''}$ for each node $m''\in Adj_{n'}\setminus\{m'\}$ by item~\ref{dfitem:different child label} of Definition~\ref{df:RBR graph}.
Then, $hd(\rho)\neq\ell_{m''}$ for each node $m''\in Adj_{n'}\setminus\{m'\}$ by statement~\eqref{eq:Mar-1-4}.
Hence, by Lemma~\ref{lm:March-1-b},
\begin{equation}\label{eq:Mar-1-5}
\rho\notin\Pi_{m''}^{k}
\end{equation}
for each integer $k$ such that $0<k\leq i-1$ and each node $m''\in Adj_{n'}\setminus\{m'\}$.
In a word, $\rho\notin\Pi_{\bar{m}}^{k}$ for each integer $k$ such that $0<k\leq i-1$ and each node $\bar{m}\in Adj_{n'}$ by statements~\eqref{eq:Mar-1-6} and \eqref{eq:Mar-1-5}.
Then, by statement~\eqref{eq:path set}, for each integer $k$ such that $0<k\leq i-1$,
\begin{equation}\label{eq:Mar-1-7}
\ell_{n'}\!::\!\rho\notin\Pi_{n'}^{k+1}.
\end{equation}
Note that $|\rho|=j$ by statement~\eqref{eq:Mar-1-3} and Lemma~\ref{lm:March-1-a}. Then, $|\rho|\geq 1$ because $0<j\leq i-1$. Thus, $|\ell_{n'}\!::\!\rho|\geq 2$.
Hence, $\ell_{n'}\!::\!\rho\notin\Pi_{n'}^{1}$ by Lemma~\ref{lm:March-1-a}.
Then, $\ell_{n'}\!::\!\rho\notin\Pi_{n'}^{k}$ for each integer $k$ such that $0<k\leq i$ by statement~\eqref{eq:Mar-1-7}.
Thus, $\ell_{n'}\!::\!\rho\notin\Psi_{n'}^{i}$ by statement~\eqref{eq:accumulated path set}.
Therefore, $\Psi_{n}^{i}\neq\Psi_{n'}^{i}$ by statements~\eqref{eq:3-11-1} and \eqref{eq:Mar-1-9}, which contradicts the assumption $\Psi_{n}^{i}=\Psi_{n'}^{i}$ of the lemma.
\end{proof}

\begin{lemma}\label{lm:4-4-a}
For any nodes $m\in Adj_n$ and $m'\in Adj_{n'}$, if $\Psi_{n}^{*}=\Psi_{n'}^{*}$  and $\ell_m=\ell'_{m'}$, then $\Psi_{m}^{*}=\Psi_{m'}^{*}$.
\end{lemma}
\begin{proof}
Note that, by Lemma~\ref{lm:equal path set} and Lemma~\ref{lm:equal whole path set}, the assumption $\Psi_{n}^{*}=\Psi_{n'}^{*}$ implies that $\Psi_{n}^{i}=\Psi_{n'}^{i}$ for each integer $i\geq 1$.
Then, the assumptions $m\in Adj_n$, $m'\in Adj_{n'}$, and $\ell_m=\ell'_{m'}$ imply that $\Psi_{m}^{i-1}=\Psi_{m'}^{i-1}$ for each integer $i\geq 1$ by Lemma~\ref{lm:Feb-21-a}.
Thus, $\Psi_{m}^{*}=\Psi_{m'}^{*}$ again by Lemma~\ref{lm:equal path set} and Lemma~\ref{lm:equal whole path set}.
\end{proof}

Lemma~\ref{lm:3-15-a} below shows that, if a ``belief similarity'' exists between the agents denoted by nodes $n$ and $n'$, then the sets of rational agents should be the same in their beliefs and a rational agent $b$ should also have similar beliefs.

\begin{lemma}\label{lm:3-15-a}
For any integer $i$ such that $i\geq 2$, if $\Psi_n^i=\Psi_{n'}^i$, then  $\mathcal{B}_n=\mathcal{B}_{n'}$ and $\Psi_{\beta_n(b)}^{i-1}=\Psi_{\beta'_n(b)}^{i-1}$ for each agent $b\in\mathcal{B}_n$.
\end{lemma}
\begin{proof}
Note that, by Lemma~\ref{lm:equal path set} and the assumption $\Psi_n^i=\Psi_{n'}^i$ where $i\geq 2$,
\begin{equation}\label{eq:Mar-4-1}
\Pi_n^2=\Pi_{n'}^2.
\end{equation}
Meanwhile, by statements~\eqref{eq:path set},
\begin{equation}\notag
\begin{aligned}
\Pi_n^2&=\{(\ell_n,\ell_m)\,|\,m\in Adj_n\};\\
\Pi_{n'}^2&=\{(\ell'_{n'},\ell'_{m'})\,|\,m'\in Adj_{n'}\}.
\end{aligned}
\end{equation}
Then, $\{\ell_m\,|\,m\in Adj_n\}\!=\!\{\ell'_{m'}\,|\,m'\in Adj_{n'}\}$ by statement~\eqref{eq:Mar-4-1},
which means $\mathcal{B}_n=\mathcal{B}_{n'}$ by statement~\eqref{eq:1 step agent set}.
Thus, by Definition~\ref{df:function tau}, for each agent $b\in\mathcal{B}_n$,
\begin{equation}\notag
\beta_n(b)\in Adj_n,\ \beta_{n'}(b)\in Adj_{n'},\ \text{and } \ell_{\beta_n(b)}=b=\ell'_{\beta_{n'}(b)}.
\end{equation}
Hence, $\Psi^{i-1}_{\beta_{n}(b)}=\Psi^{i-1}_{\beta_{n'}(b)}$ for each agent $b\in\mathcal{B}_n$ by Lemma~\ref{lm:Feb-21-a} and the assumption $\Psi^{i}_{n}=\Psi^{i}_{n'}$.
\end{proof}

Using the property above, the lemma below proves that, for each integer $i\geq 1$, the results of the $i^{th}$ iterative rationalisation are identical on the nodes $n$ and $n'$ if $\Psi_n^{i}=\Psi_{n'}^{i}$.

\begin{lemma}\label{lm:indistinguishable nodes are doxastic equivalent}
For any integer $i\ge 1$, if $\Psi_n^{i}=\Psi_{n'}^{i}$, then $\mathbb{R}_B^i(S^\Delta)_n=\mathbb{R}_{B'}^i(S^\Delta)_{n'}$ for each game $G=(\Delta,\preceq)$. 
\end{lemma}
\begin{proof}
By Lemma~\ref{lm:3-11-a}, if $\ell_n\neq\ell'_{n'}$, then $\Psi_n^{i}\neq\Psi_{n'}^{i}$ for each integer $i\geq 1$, which contradicts the assumption $\Psi_n^{i}=\Psi_{n'}^{i}$. Thus,
\begin{equation}\label{eq:Feb-20-6}
\ell_n=\ell'_{n'}.
\end{equation}

We prove this lemma by induction on integer $i$.
In the base case where $i=1$, by Definition~\ref{df:ith rationalisation} and Definition~\ref{df:rationalisation on solution},
\begin{equation}\label{eq:Feb-20-7}
\begin{aligned}
\mathbb{R}_B^1(S^\Delta)_n&=\Re_{\ell_n}(\tilde{\Theta}_n(S^\Delta));\\
\mathbb{R}_{B'}^1(S^\Delta)_{n'}&=\Re_{\ell'_{n'}}(\tilde{\Theta}_{n'}(S^\Delta)).
\end{aligned}
\end{equation}
Note that, by Definition~\ref{df:belief scene}, statements~\eqref{eq:solution Delta} and \eqref{eq:Feb-20-6},
\begin{equation}\notag
\tilde{\Theta}_n(S^\Delta)=\prod_{b\neq\ell_n}\Delta_b=\prod_{b\neq\ell'_{n'}}\Delta_b=\tilde{\Theta}_{n'}(S^\Delta).
\end{equation}
Then, $\Re_{\ell_n}(\tilde{\Theta}_n(S^\Delta))=\Re_{\ell'_{n'}}(\tilde{\Theta}_{n'}(S^\Delta))$ by statement~\eqref{eq:Feb-20-6}.
Thus, $\mathbb{R}_B^1(S^\Delta)_n=\mathbb{R}_{B'}^1(S^\Delta)_{n'}$ by statement~\eqref{eq:Feb-20-7}.

In the case where $i\geq 2$, by Lemma~\ref{lm:3-15-a} and the assumption $\Psi_n^{i}=\Psi_{n'}^{i}$,
\begin{equation}\label{eq:3-15-19}
\mathcal{B}_n=\mathcal{B}_{n'}
\end{equation}
and $\Psi^{i-1}_{\beta_{n}(b)}=\Psi^{i-1}_{\beta_{n'}(b)}$ for each agent $b\in\mathcal{B}_n$.
Then, by the induction hypothesis, for each agent $b\in\mathcal{B}_n$,
\begin{equation}\label{eq:3-15-21}
\mathbb{R}_B^{i-1}(S^\Delta)_{\beta_n(b)}=\mathbb{R}_{B'}^{i-1}(S^\Delta)_{\beta_{n'}(b)}.
\end{equation}
Meanwhile, by statements~\eqref{eq:RBR graph reasoning scene}, \eqref{eq:3-15-19} and \eqref{eq:Feb-20-6}, for each agent $b\neq\ell_n$,
\begin{equation}\notag
\begin{aligned}
\tilde{\Theta}^b_{n}(\mathbb{R}_{B}^{i-1}(S^\Delta))=
\begin{cases}
\mathbb{R}_{B}^{i-1}(S^\Delta)_{\beta_n(b)}, & \text{if } b\in\mathcal{B}_n;\\
\Delta_b, & \text{otherwise};
\end{cases}\\
\tilde{\Theta}^b_{n'}(\mathbb{R}_{B'}^{i-1}(S^\Delta))=
\begin{cases}
\mathbb{R}_{B'}^{i-1}(S^\Delta)_{\beta_{n'}(b)}, & \text{if } b\in\mathcal{B}_{n};\\
\Delta_b, & \text{otherwise}.
\end{cases}
\end{aligned}
\end{equation}
Hence, by statement~\eqref{eq:3-15-21},
\begin{equation}\label{eq:3-15-20}
\tilde{\Theta}_{n}(\mathbb{R}_{B}^{i-1}(S^\Delta))=\tilde{\Theta}_{n'}(\mathbb{R}_{B'}^{i-1}(S^\Delta)).
\end{equation}
Note that, by Definition~\ref{df:rationalisation on solution} and Definition~\ref{df:ith rationalisation},
\begin{equation}\notag
\begin{aligned}
&\mathbb{R}_{B}^i(S^\Delta)_{n}=\mathbb{R}(\mathbb{R}_{B}^{i-1}(S^\Delta))_n=\Re_{\ell_n}(\tilde{\Theta}_{n}(\mathbb{R}_{B}^{i-1}(S^\Delta)));\\
&\mathbb{R}_{B'}^i(S^\Delta)_{n'}=\mathbb{R}(\mathbb{R}_{B'}^{i-1}(S^\Delta))_{n'}=\Re_{\ell'_{n'}}(\tilde{\Theta}_{n'}(\mathbb{R}_{B'}^{i-1}(S^\Delta))).
\end{aligned}
\end{equation}
Therefore, $\mathbb{R}_B^i(S^\Delta)_n=\mathbb{R}_{B'}^i(S^\Delta)_{n'}$ by statements~\eqref{eq:Feb-20-6} and \eqref{eq:3-15-20}.
\end{proof}

Informally, the above lemma shows that the iterative rationalisation process is a ``bisimulation'' on the nodes $n$ and $n'$ when they represent the same belief hierarchy.
Next, we use this finding to prove the ``if'' part of Theorem~\ref{th:doxastic equivalent nodes = indistinguishable}.

\noindent\textbf{Theorem \ref{th:doxastic equivalent nodes = indistinguishable}} (``if'' part) \textit{The nodes $n$ and $n'$ are doxastically equivalent if $\Psi_n^*=\Psi_{n'}^*$.}

\begin{proof}
By Lemma~\ref{lm:equal whole path set}, the assumption $\Psi_n^*=\Psi_{n'}^*$ implies that
\begin{equation}\label{eq:4-4-3}
\Pi_n^i=\Pi_{n'}^i \text{ for each } i\geq 1.
\end{equation}
Then, $\Pi_n^1=\Pi_{n'}^1$. Thus, by statement~\eqref{eq:path set},
\begin{equation}\label{eq:3-11-2}
\ell_n=\ell'_{n'}.
\end{equation}

Consider an arbitrary game $G=(\Delta,\preceq)$.
By Theorem~\ref{th:rational solution}, there are integers $k_1,k_2\geq 0$ such that $\mathbb{S}(G)_{n}=\mathbb{R}_{B}^{j_1}(S^\Delta)_{n}$ and $\mathbb{S}'(G)_{n'}=\mathbb{R}_{B'}^{j_2}(S^\Delta)_{n'}$ for any $j_1\geq k_1$ and $j_2\geq k_2$.
Let $k:=\max(k_1,k_2)$. Then, 
\begin{equation}\label{eq:Feb-26-5}
\mathbb{S}(G)_{n}=\mathbb{R}_{B}^{k}(S^\Delta)_{n} \text{\hspace{1mm} and \hspace{1mm}} \mathbb{S}'(G)_{n'}=\mathbb{R}_{B'}^{k}(S^\Delta)_{n'}.
\end{equation}
If $k=0$, then, by Definition~\ref{df:ith rationalisation}, statements~\eqref{eq:solution Delta} and \eqref{eq:3-11-2},
\begin{equation}\label{eq:Feb-26-6}
\mathbb{R}_{B}^{0}(S^\Delta)_{n}=\Delta_{\ell_n}=\Delta_{\ell'_{n'}}=\mathbb{R}_{B'}^{0}(S^\Delta)_{n'}.
\end{equation}
If $k\geq 1$, then $\Psi_n^k=\Psi_{n'}^k$ by statements~\eqref{eq:4-4-3} and \eqref{eq:accumulated path set}.
Thus, $\mathbb{R}_{B}^{k}(S^\Delta)_{n}=\mathbb{R}_{B'}^{k}(S^\Delta)_{n'}$ by Lemma~\ref{lm:indistinguishable nodes are doxastic equivalent}.
Hence, it is always true that $\mathbb{S}(G)_{n}=\mathbb{S}'(G)_{n'}$ by statement~\eqref{eq:Feb-26-5} and \eqref{eq:Feb-26-6}.
Therefore, the ``if'' part statement of the theorem follows from Definition~\ref{df:doxastic equivalent nodes}.
\end{proof}

\section{Equivalence Between RBR Graphs}\label{sec:app RBR graph equivalence}

In this section, we formally prove Theorem~\ref{th:RBR graph equivalent}, which states the necessary and sufficient condition for the equivalence of two RBR graphs.
After that, we discuss the structural similarity between equivalent RBR graphs, which is used later to design an algorithm that minimises an RBR graph.

For the convenience of proof, we add a few new notations and rewrite Theorem~\ref{th:RBR graph equivalent} in the following form.

\noindent\textbf{Theorem \ref{th:RBR graph equivalent}} \textit{The two RBR graphs $B=(N,E,\ell,\pi)$ and $B'=(N',E',\ell',\pi')$ are equivalent if and only if, for each agent $a\in\mathcal{A}$, \textbf{either}
\begin{enumerate}[label={C\arabic*.}, ref=C\arabic*,left=0pt]
\item both $\pi_a$ and $\pi'_a$ are not defined, \textbf{or}\label{item:Feb_27-1}
\item $\pi_a$ and $\pi'_a$ are both defined and doxastically equivalent.\label{item:Feb_27-2}
\end{enumerate}}
\begin{proof}
For the ``if'' part, consider an arbitrary agent $a\in\mathcal{A}$ and an arbitrary game $G=(\Delta,\preceq)$.
If condition~\ref{item:Feb_27-1} above is true, then $\mathfrak{R}(G)_a=\Delta_{a}=\mathfrak{R}'(G)_a$ by Definition~\ref{df:doxastic rationalisability}.
If condition~\ref{item:Feb_27-2} above is true, then, by Definition~\ref{df:doxastic rationalisability},
\begin{equation}\label{eq:Feb-27-1}
\mathfrak{R}(G)_a=\mathbb{S}(G)_{\pi_a} \text{\hspace{1mm} and \hspace{1mm}} \mathfrak{R}'(G)_a=\mathbb{S}'(G)_{\pi'_a}
\end{equation}
Note that, $\mathbb{S}(G)_{\pi_a}=\mathbb{S}'(G)_{\pi'_a}$ by Definition~\ref{df:doxastic equivalent nodes} and the assumption in condition~\ref{item:Feb_27-2} that $\pi_a$ and $\pi'_a$ are doxastic equivalent.
Then, $\mathfrak{R}(G)_a=\mathfrak{R}'(G)_a$ by statement~\eqref{eq:Feb-27-1}.
In conclusion, if either of the conditions~\ref{item:Feb_27-1} and \ref{item:Feb_27-2} above is true, then $\mathfrak{R}(G)_a=\mathfrak{R}'(G)_a$.
Thus, the RBR graphs $B$ and $B'$ are equivalent by Definition~\ref{df:equivalent} because agent $a$ is an arbitrary one in set $\mathcal{A}$.

For the ``only if'' part, consider the case where there is an agent $a\in\mathcal{A}$ such that both conditions~\ref{item:Feb_27-1} and \ref{item:Feb_27-2} are false.
Then, one of the next two statements must be true:
\begin{enumerate}[label={S\arabic*.}, ref=S\arabic*,left=0pt]
\item only one of $\pi_a$ and $\pi'_a$ is defined.\label{st:Feb-27-1}
\item  both $\pi_a$ and $\pi'_a$ are defined but they are not doxastically equivalent.\label{st:Feb-27-2}
\end{enumerate}

Then, the ``only if'' part of the statement of the theorem follows from Claim~\ref{cl:Feb-27-a} and Claim~\ref{cl:Feb-27-b} below.

\begin{claim}\label{cl:Feb-27-a}
If statement~\ref{st:Feb-27-1} above is true, then the RBR graphs $B$ and $B'$ are not equivalent.
\end{claim}
\begin{proof-of-claim}
Without loss of generality, suppose that $\pi_a$ is defined and $\pi'_a$ is not defined.
Then, by Definition~\ref{df:doxastic rationalisability},
\begin{equation}\label{eq:Feb-26-8}
\mathfrak{R}(G)_a=\mathbb{S}(G)_{\pi_a} \text{\hspace{1mm} and \hspace{1mm}} \mathfrak{R}'(G)_a=\Delta_a
\end{equation}
for each game $G=(\Delta,\preceq)$.
Consider a game $G_{0,1}$ where every agent can choose between two strategies $0$ and $1$ but a rational agent always chooses strategy $1$. 
Formally, game $G_{0,1}$ is a $(\Delta,\prec)$ where, for each agent $b\in\mathcal{A}$,
\begin{itemize}
\item $\Delta_b=\{0,1\}$;
\item $\s\preceq_b \s'$ if and only if $s_b\leq s'_b$ for each pair of outcomes $\s,\s'\in\prod_{c\in\mathcal{A}}\Delta_c$.
\end{itemize}
Then, $0\lhd_{\Theta_a}1$ for each reasoning scene $\Theta_a$ in game $G_{0,1}$ by Definition~\ref{df:dominated strategy}.
Thus, $\Re_a(\Theta_a)=\{1\}$ for each reasoning scene $\Theta_a$ by Definition~\ref{df:rational response}.
Note that $\ell_{\pi_a}=a$ by item~\ref{dfitem:RBR graph designating function} of Definition~\ref{df:RBR graph}.
Then, $\mathbb{S}(G_{0,1})_{\pi_a}=\{1\}$ by Definition~\ref{df:rationalisation on solution}, Definition~\ref{df:ith rationalisation}, and Theorem~\ref{th:rational solution}.
Hence, $\mathfrak{R}_a(G_{0,1})=\{1\}$ and $\mathfrak{R}'_a(G_{0,1})=\{0,1\}$ by statement~\eqref{eq:Feb-26-8} and the definition of game $G_{0,1}$.
In other word, there is an agent $a$ and a game $G_{0,1}$ such that $\mathfrak{R}_a(G_{0,1})\neq\mathfrak{R}'_a(G_{0,1})$.
Therefore, the RBR graphs $B$ and $B'$ are not equivalent by Definition~\ref{df:equivalent}.
\end{proof-of-claim}

\begin{claim}\label{cl:Feb-27-b}
If statement~\ref{st:Feb-27-2} above is true, then the RBR graphs $B$ and $B'$ are not equivalent.
\end{claim}
\begin{proof-of-claim}
By Definition~\ref{df:doxastic rationalisability} and the assumption that both $\pi_a$ and $\pi'_a$ are defined in statement~\ref{st:Feb-27-2},
\begin{equation}\label{eq:Feb-27-2}
\mathfrak{R}(G)_a=\mathbb{S}(G)_{\pi_a} \text{\hspace{1mm} and \hspace{1mm}} \mathfrak{R}'(G)_a=\mathbb{S}'(G)_{\pi'_a}
\end{equation}
for each game $G$. Meanwhile, by Definition~\ref{df:doxastic equivalent nodes} and the assumption that $\pi_a$ and $\pi'_a$ are not doxastically equivalent in statement~\ref{st:Feb-27-2}, there is a game $G$ where $\mathbb{S}(G)_{\pi_a}\neq\mathbb{S}'(G)_{\pi'_a}$.
Then, $\mathfrak{R}(G)_a\neq\mathfrak{R}'(G)_a$ in such a game $G$ by statement~\eqref{eq:Feb-27-2}.
Hence, the RBR graphs $B$ and $B'$ are not equivalent by Definition~\ref{df:equivalent}.
\end{proof-of-claim}
This concludes the proof of the theorem.
\end{proof}

By Theorem~\ref{th:doxastic equivalent nodes = indistinguishable} and Theorem~\ref{th:RBR graph equivalent}, we can easily conclude
Corollary~\ref{cr:RBR graph equivalent}.
It states the necessary and sufficient condition for the equivalence of two RBR graphs based on the structural property of the graphs.
We rewrite Corollary~\ref{cr:RBR graph equivalent} in the next form for the convenience of later proofs.

\noindent\textbf{Corollary~\ref{cr:RBR graph equivalent}} \textit{The two RBR graphs $B=(N,E,\ell,\pi)$ and $B'=(N',E',\ell',\pi')$ are equivalent if and only if
\begin{enumerate}
\item $\pi$ and $\pi'$ have the same domain $\mathcal{D}$ of definition, \textbf{and}
\item $\Psi_{\pi_a}^*=\Psi_{\pi'_a}^*$ for each agent $a\in\mathcal{D}$.\label{item:cr:RBR graph equivalent 2}
\end{enumerate}}

In preparation for introducing the algorithm that minimises an RBR graph, we need to investigate more (structural) properties of RBR graph equivalence.
They are also important and interesting but we lack space in the main text to introduce them.
We formally state and prove them in the next theorem and the following lemmas.

Recall that set $\Psi_n^*$ captures the belief hierarchy of the (real or doxastic) agent denoted by node $n$.
The next theorem shows that a belief hierarchy must show in both or neither of two equivalent RBR graphs.
Note that, compared with item~\ref{item:cr:RBR graph equivalent 2} of Corollary~\ref{cr:RBR graph equivalent} where only real agents are considered, Theorem~\ref{th:node type matches in equivalent RBR graphs} also consider the doxastic agents.

\begin{theorem}\label{th:node type matches in equivalent RBR graphs}
If RBR graphs $(N,E,\ell,\pi)$ and $(N',E',\ell',\pi')$ are equivalent, then $\{\Psi_n^*\mid n\in N\}=\{\Psi_{n'}^*\mid n'\in N'\}$.
\end{theorem}
\begin{proof}
Without loss of generality, it suffices to prove that $\{\Psi_n^*\mid n\in N\}\subseteq\{\Psi_{n'}^*\mid n'\in N'\}$.
Consider an arbitrary node $n\in N$.
Note that, by item~\ref{dfitem:RBR graph no irrelevant dummy node} of Definition~\ref{df:RBR graph}, there is an agent $a\in\mathcal{A}$ and a path $(n_1,\dots,n_k)$ in the RBR graph $(N,E,\ell,\pi)$ such that
\begin{equation}\label{eq:3-15-8}
n_1=\pi_a,
\end{equation}
\begin{equation}\label{eq:3-18-4}
n_k=n,
\end{equation}
and, for each integer $j$ such that $1<j\leq k$,
\begin{equation}\label{eq:3-15-15}
n_{j}\in Adj_{n_{j-1}}.
\end{equation}
Then, by statements~\eqref{eq:path set} and \eqref{eq:3-15-8},
\begin{equation}\label{eq:3-15-10}
(\ell_{n_1},\dots,\ell_{n_k})\in\Pi_{\pi_a}^k.
\end{equation}
Consider the designated node $\pi'_a\in N'$.
Then, by item~\ref{item:cr:RBR graph equivalent 2} of Corollary~\ref{cr:RBR graph equivalent} and the assumption of the theorem,
\begin{equation}\label{eq:3-15-11}
\Psi^*_{\pi_a}=\Psi^*_{\pi'_a}.
\end{equation}
Thus, $\Pi^k_{\pi_a}\!=\!\Pi^k_{\pi'_a}$ by Lemma~\ref{lm:equal whole path set}.
Then, $(\ell_{n_1},\dots,\ell_{n_k})\!\in\!\Pi_{\pi'_a}^k$ by statement~\eqref{eq:3-15-10}.
Hence, by Lemma~\ref{lm:March-1-c}, there is a path $(n'_1,\dots,n'_k)$ in the RBR graph $(N',E',\ell',\pi')$, where 
\begin{equation}\label{eq:3-15-12}
n'_1=\pi'_a,  
\end{equation}
for each integer $j$ such that $1<j\leq k$,
\begin{equation}\label{eq:3-15-16}
n'_{j}\in Adj_{n'_{j-1}},
\end{equation}
and, for each integer $j$ such that $1\leq j\leq k$,
\begin{equation}\label{eq:3-15-13}
\ell'_{n'_j}=\ell_{n_j}.
\end{equation}
Note that, by statements~\eqref{eq:3-15-11}, \eqref{eq:3-15-8} and \eqref{eq:3-15-12},
\begin{equation}\label{eq:3-15-14}
\Psi^*_{n_1}=\Psi^*_{\pi_a}=\Psi^*_{\pi'_a}=\Psi^*_{n'_1}.
\end{equation}

\begin{claim}\label{cl:3-15-b}
$\Psi^*_{n_j}=\Psi^*_{n'_j}$ for each integer $j\leq k$.
\end{claim}
\begin{proof-of-claim}
We prove the statement of the claim by induction on $j$. 
In the base case where $j=1$, the statement of the claim follows from statement~\eqref{eq:3-15-14}.

In the case where $2\leq j\leq k$, by the induction hypothesis, $\Psi^*_{n_{j-1}}=\Psi^*_{n'_{j-1}}$.
Then, $\Psi^*_{n_j}=\Psi^*_{n'_j}$ by statements~\eqref{eq:3-15-15}, \eqref{eq:3-15-16}, \eqref{eq:3-15-13}, and Lemma~\ref{lm:4-4-a}.
\end{proof-of-claim}

Note that $\Psi^{*}_{n_k}=\Psi^{*}_{n'_k}$ by Claim~\ref{cl:3-15-b}.
Then, $\Psi^{*}_{n}=\Psi^{*}_{n'_k}$ by statement~\eqref{eq:3-18-4}.
Hence, $\Psi^{*}_{n}\in\{\Psi^{*}_{n'}\mid n'\in N'\}$. 
\end{proof}

Next, we define the notion of local isomorphism which describes a local structural similarity between two RBR graphs.
It is an extension of {\em covering map} in graph theory \cite{angluin1980local}. In Definition~\ref{df:RBR graph isomorphic}, items~\ref{dfitem:isomorphic edge} and \ref{dfitem:isomorphic label} capture the essence of covering map for the general labelled graphs, while item~\ref{dfitem:isomorphic designation} is specific to RBR graphs.

\begin{definition}\label{df:RBR graph isomorphic}
A \textbf{local isomorphism} from RBR graph $(N,E,\ell,\pi)$ to RBR graph $(N',E',\ell',\pi')$ is a surjective function $\alpha: N\to N'$ such that
\begin{enumerate}
\item $\{\alpha(m)\mid nEm\}=\{m'\mid\alpha(n)E'm'\}$ for each $n\in N$;\label{dfitem:isomorphic edge}
\item $\ell_n=\ell'_{\alpha(n)}$ for each node $n\in N$;\label{dfitem:isomorphic label}
\item $\pi$ and $\pi'$ have the same domain $\mathcal{D}$ of definition and $\alpha(\pi_a)=\pi'_a$ for each agent $a\in\mathcal{D}$.\label{dfitem:isomorphic designation}
\end{enumerate}

Furthermore, $\alpha$ is called \textbf{isomorphism} if it is a bijective function (\textit{i.e.} $|N|=|N'|$).
\end{definition}

If there is a (local) isomorphism from the RBR graph $B$ to the RBR graph $B'$, then we say that the RBR graph $B$ is ({\bf\em locally}) {\bf\em isomorphic} to the RBR graph $B'$.
Note that, in the case that $\alpha$ is a bijective function, items~\ref{dfitem:isomorphic edge} and \ref{dfitem:isomorphic label} of Definition~\ref{df:RBR graph isomorphic} capture the isomorphism between two labelled graphs \cite{hsieh2006efficient}.

The next two lemmas give the reason why we study local isomorphism: it is related to the doxastic equivalence between nodes and the equivalence between RBR graphs.
They are used in the next section to prove the correctness of the algorithm that minimises an RBR graph.

\begin{lemma}\label{lm:path set equal in isomorphic}
If $\alpha$ is a local isomorphism from RBR graph $(N,E,\ell,\pi)$ to $(N',E',\ell',\pi')$, then $\Psi_{n}^*=\Psi_{\alpha(n)}^*$ for each node $n\in N$.
\end{lemma}
\begin{proof}
By Lemma~\ref{lm:equal whole path set}, it suffices to prove by induction that $\Pi_{n}^i=\Pi_{\alpha(n)}^i$ for each integer $i\geq 1$.
In the base case $i=1$, the statement $\Pi_{n}^1=\Pi_{\alpha(n)}^1$ follows from statement~\eqref{eq:path set} and item~\ref{dfitem:isomorphic label} of Definition~\ref{df:RBR graph isomorphic}.

Next, we consider the case $i\geq 2$.
By statement~\eqref{eq:path set},
\begin{align}
\Pi_n^i&=\bigcup_{m\in Adj_n}\{\ell_n\!::\!\sigma\mid \sigma\in\Pi_{m}^{i-1}\};\label{eq:Mar-3-1}\\
\Pi_{\alpha(n)}^i&=\bigcup_{m'\in Adj_{\alpha(n)}}\{\ell'_{\alpha(n)}\!::\!\sigma\mid \sigma\in\Pi_{m'}^{i-1}\}.\label{eq:Mar-3-2}
\end{align}
Meanwhile, $Adj_{\alpha(n)}=\{\alpha(m)\mid m\in Adj_n\}$ by statement~\eqref{eq:1 step neighbour set} and item~\ref{dfitem:isomorphic edge} of Definition~\ref{df:RBR graph isomorphic}.
Then, by item~\ref{dfitem:isomorphic label} of Definition~\ref{df:RBR graph isomorphic} and statement~\eqref{eq:Mar-3-2},
\begin{equation}\label{eq:Mar-3-4}
\Pi_{\alpha(n)}^i=\bigcup_{m\in Adj_n}\{\ell_{n}\!::\!\sigma\mid\sigma\in\Pi_{\alpha(m)}^{i-1}\}.
\end{equation}
Note that, by the induction hypothesis, $\Pi_{m}^{i-1}=\Pi_{\alpha(m)}^{i-1}$ for each node $m\in Adj_n$.
Then, $\Pi_n^i=\Pi_{\alpha(n)}^i$ by statements~\eqref{eq:Mar-3-1} and \eqref{eq:Mar-3-4}.
\end{proof}

\begin{lemma}\label{lm:isomorphic equivalent}
Two RBR graphs are equivalent if one of them is locally isomorphic to the other.
\end{lemma}
\begin{proof}
Consider the case that an RBR graph $(N,E,\ell,\pi)$ is locally isomorphic to another RBR graph $(N',E',\ell',\pi')$. By Definition~\ref{df:RBR graph isomorphic}, let $\alpha$ be the local isomorphism and $\mathcal{D}$ be the common domain of definition of $\pi$ and $\pi'$.
Then, by Lemma~\ref{lm:path set equal in isomorphic}, for each agent $a\in\mathcal{D}$,
\begin{equation}\label{eq:3-15-1}
\Psi_{\pi_a}^*=\Psi_{\alpha(\pi_a)}^*.
\end{equation}
Meanwhile, $\alpha(\pi_a)=\pi'_a$ for each agent $a\in\mathcal{D}$ by item~\ref{dfitem:isomorphic designation} of Definition~\ref{df:RBR graph isomorphic}.
Then, $\Psi_{\pi_a}^*=\Psi_{\pi'_a}^*$ for each agent $a\in\mathcal{D}$ by statement~\eqref{eq:3-15-1}.
Thus, the RBR graphs $(N,E,\ell,\pi)$ and $(N',E',\ell',\pi')$ are equivalent by Corollary~\ref{cr:RBR graph equivalent} because $\mathcal{D}$ is the common domain of definition of $\pi$ and $\pi'$.
\end{proof}

\section{Minimisation of RBR Graphs}

In this section, we clarify the minimality of an RBR graph. In addition, we prove the correctness of Algorithm~\ref{alg:compute CF} in Section~\ref{sec:minimisation RBR graph} and analyse its time complexity.

\subsection{Canonical RBR Graph}\label{sec:app canonical RBR graph}

Definition~\ref{df:canonical RBR graph} specifies which type of RBR graph is canonical.
The next lemma follows directly from Definition~\ref{df:canonical RBR graph}.
\begin{lemma}\label{lm:node type differ in canonical graph}
$|\{\Psi^*_{n}\mid n\in N\}|=|N|$ in a canonical RBR graph $(N,E,\ell,\pi)$.
\end{lemma}

The next lemma shows that canonical RBR graphs are minimal in terms of the number of nodes.
Briefly speaking, the statement of the lemma follows from Lemma~\ref{lm:node type differ in canonical graph} and Theorem~\ref{th:node type matches in equivalent RBR graphs} by the {\em pigeonhole principle}.
We give a formal definition below.

\begin{lemma}\label{lm:canonical is minimal}
A canonical RBR graph is not equivalent to any RBR graph with fewer nodes.
\end{lemma}
\begin{proof}
Consider an arbitrary belief $(N',E',\ell',\pi')$ that is equivalent to a canonical RBR graph $(N,E,\ell,\pi)$.
Then,
\begin{equation}\notag
|\{\Psi_{n}^*\mid n\in N\}|=|N|
\end{equation}
by Lemma~\ref{lm:node type differ in canonical graph} and
\begin{equation}\notag
\{\Psi_n^*\mid n\in N\}=\{\Psi_{n'}^*\mid n'\in N'\}
\end{equation}
by Theorem~\ref{th:node type matches in equivalent RBR graphs}.
Thus,
\begin{equation}\label{eq:4-4-4}
|\{\Psi_{n'}^*\mid n'\in N'\}|=|N|.
\end{equation}
Note that, $|\{\Psi_{n'}^*\mid n'\in N'\}|\leq |N'|$.
Hence, $|N|\leq |N'|$ by statement~\eqref{eq:4-4-4}.
\end{proof}

The above lemma allows us to minimise an RBR graph by finding an equivalent canonical RBR graph.
The next theorem states a stronger result about two equivalent canonical RBR graphs.
It shows the structural uniqueness of the equivalent canonical RBR graph of any RBR graph.
In other words, the most condensed expression (with RBR graphs) of any RBR system is unique.

\begin{theorem}\label{th:equivalent canonical graph are isomorphic}
If two canonical RBR graphs $(N,E,\ell,\pi)$ and $(N',E',\ell',\pi')$ are equivalent, then they are isomorphic.
\end{theorem}
\begin{proof}
By Theorem~\ref{th:node type matches in equivalent RBR graphs} and the assumption of the theorem,
\begin{equation}\label{eq:4-5-1}
\{\Psi_{n}^*\mid n\in N\}=\{\Psi_{n'}^*\mid n'\in N'\}.
\end{equation}
Then, by Lemma~\ref{lm:node type differ in canonical graph},
\begin{equation}\notag
|N|=|\{\Psi_{n}^*\mid n\in N\}|=|\{\Psi_{n'}^*\mid n'\in N'\}|=|N'|.
\end{equation}
Hence, there is a bijective function $\alpha: N\to N'$ such that
\begin{equation}\label{eq:4-5-2}
\Psi^*_{n}=\Psi^*_{\alpha(n)}
\end{equation}
for each node $n\in N$ by statement~\eqref{eq:4-5-1}.

\begin{claim}\label{cl:4-5-a}
$\{\alpha(m)\mid nEm\}=\{m'\mid\alpha(n)E'm'\}$ for each node $n\in N$.
\end{claim}
\begin{proof-of-claim}
By statements~\eqref{eq:1 step neighbour set}, \eqref{eq:1 step agent set}, and Definition~\ref{df:function tau},
\begin{align}
\{m\mid nEm\}&=\{\beta_n(b)\mid b\in\mathcal{B}_n\};\label{eq:4-9-1}\\
\{m'\mid\alpha(n)E'm'\}&=\{\beta_{\alpha(n)}(b)\mid b\in\mathcal{B}_{\alpha(n)}\}.\label{eq:4-9-2}
\end{align}
By applying function $\alpha$ to both sides of statement~\eqref{eq:4-9-1},
\begin{equation}\label{eq:4-9-3}
\{\alpha(m)\mid nEm\}=\{\alpha(\beta_n(b))\mid b\in\mathcal{B}_n\}.
\end{equation}

Note that, $\Psi^2_{n}=\Psi^2_{\alpha(n)}$ by Lemma~\ref{lm:equal whole path set}, statements~\eqref{eq:accumulated path set} and \eqref{eq:4-5-2}.
Then, $\mathcal{B}_n=\mathcal{B}_{\alpha(n)}$ by Lemma~\ref{lm:3-15-a}.
Thus, by statements~\eqref{eq:4-9-2} and \eqref{eq:4-9-3}, it suffices to prove for each agent $b\in\mathcal{B}_n$ that $\alpha(\beta_n(b))=\beta_{\alpha(n)}(b)$.

Consider an arbitrary agent $b\in\mathcal{B}_n$.
Then, $\beta_n(b)\in Adj_n$, $\beta_{\alpha(n)}(b)\in Adj_{\alpha(n)}$, and $\ell_{\beta_n(b)}=b=\ell'_{\beta_{\alpha(n)}(b)}$ by statements~\eqref{eq:1 step neighbour set}, \eqref{eq:1 step agent set}, and Definition~\ref{df:function tau}.
Thus, $\Psi^*_{\beta_n(b)}=\Psi^*_{\beta_{\alpha(n)}(b)}$ by statement~\eqref{eq:4-5-2} and Lemma~\ref{lm:4-4-a}.
Hence, by statement~\eqref{eq:4-5-2},
\begin{equation}\label{eq:4-5-5}
\Psi^*_{\alpha(\beta_n(b))}=\Psi^*_{\beta_{\alpha(n)}(b)}.
\end{equation}
Note that both nodes $\alpha(\beta_n(b))$ and $\beta_{\alpha(n)}(b)$ are in set $ N'$.
Therefore, $\alpha(\beta_n(b))=\beta_{\alpha(n)}(b)$ by statement~\eqref{eq:4-5-5}, Definition~\ref{df:canonical RBR graph}, and the assumption that the RBR graph $(N',E',\ell',\pi')$ is canonical.
\end{proof-of-claim}

\begin{claim}\label{cl:4-5-b}
$\ell_{n}=\ell'_{\alpha(n)}$ for each node $n\in N$.
\end{claim}
\begin{proof-of-claim}
Note that $\Pi^1_{n}=\Pi^1_{\alpha(n)}$ by statement~\eqref{eq:4-5-2} and Lemma~\ref{lm:equal whole path set}.
Then, $\ell_{n}=\ell'_{\alpha(n)}$ by statement~\eqref{eq:path set}.
\end{proof-of-claim}

\begin{claim}\label{cl:4-5-c}
$\pi$ and $\pi'$ have the same domain $\mathcal{D}$ of definition and $\alpha(\pi_a)=\pi'_a$ for each agent $a\in\mathcal{D}$.
\end{claim}
\begin{proof-of-claim}
By Corollary~\ref{cr:RBR graph equivalent} and the assumption of equivalence, $\pi$ and $\pi'$ have the same domain $\mathcal{D}$ of definition and $\Psi^*_{\pi_a}=\Psi^*_{\pi'_a}$ for each agent $a\in\mathcal{D}$ .
Thus, $\Psi^*_{\alpha(\pi_a)}=\Psi^*_{\pi'_a}$ by statement~\eqref{eq:4-5-2}.
Hence, $\alpha(\pi_a)=\pi'_a$ by Definition~\ref{df:canonical RBR graph} because both $\alpha(\pi_a)$ and $\pi'_a$ are nodes in the canonical RBR graph $(N',E',\ell',\pi')$.
\end{proof-of-claim}

Recall that $\alpha$ is a bijective function. Then, the statement of the theorem follows from Claim~\ref{cl:4-5-a}, Claim~\ref{cl:4-5-b}, and Claim~\ref{cl:4-5-c} by Definition~\ref{df:RBR graph isomorphic}.
\end{proof}

\subsection{Partition Sequence of the Node Set}\label{sec:app partition sequence}

The core of Algorithm~\ref{alg:compute CF} is a sequence of partitions of the node set of an RBR graph.
This is inspired by Myhill-Nerode theorem~\cite{myhill1957finite,nerode1958linear}.
In this subsection, we formally define the partition sequence and prove its properties, based on which we prove the correctness of Algorithm~\ref{alg:compute CF} in the next subsection.

Throughout this subsection, we consider an arbitrary RBR graph $(N,E,\ell,\pi)$ and a sequence of \textit{partitions} $\mathbb{P}^1$, $\mathbb{P}^2$, $\mathbb{P}^3,\dots$ of the set $N$ of nodes, where
\begin{equation}\label{eq:1st partition in minimisation}
\mathbb{P}^1:=\big\{\{n'\in N\,|\,\ell_{n'}=\ell_n\}\,\big|\,n\in N\big\}
\end{equation}
is the \textit{partition} of set $N$ according to the labels of the nodes, and,
for each integer $i\geq 2$,
\begin{equation}\label{eq:ith partition in minimisation}
\hspace{-4mm}\mathbb{P}^i:=\!\!\!\!\bigcup_{P\in \mathbb{P}^{i-1}}\!\!\!\! \big\{\{n'\in P\mid type^i(n')=type^i(n) \}\mid n\in P\big\}\hspace{-4mm}
\end{equation}
where
\begin{equation}\label{eq:partition type in minimisation}
type^i(n):=\{P\in\mathbb{P}^{i-1}\mid nEm,m\in P\}.
\end{equation}
Observe that, for each integer $i\geq 2$, the family $\mathbb{P}^i$ is a \textit{refined} partition of the partition $\mathbb{P}^{i-1}$.
More precisely, the refining process first partitions each set $P\in\mathbb{P}^{i-1}$ according to some \textit{type} information of the nodes and then combines the partitions for all sets $P\in\mathbb{P}^{i-1}$ as the family $\mathbb{P}^i$.
Let
\begin{equation}\label{eq:finest partition in minimisation}
\mathbb{P}:=\lim_{i\to\infty}\mathbb{P}^i.
\end{equation}
As we will see in Lemma~\ref{lm:finest partition exists}, such a partition exists and it is the \textbf{finest partition} in the sequence.

For each node $n\in N$ and each integer $i\geq 1$, we abuse\footnote{Mathematically, the partition $\mathbb{P}^i$ is an equivalence relation on the set $N$ of nodes and each set $P\in\mathbb{P}^i$ is an equivalence class. Then, $\mathbb{P}^i(n)$ is the equivalence class containing node $n$. Thus, $\mathbb{P}^i(n)=\mathbb{P}^i(n')$ if and only if the nodes $n,n'$ belong to the same equivalence class.} the notation $\mathbb{P}^i(n)$ to denote the set $P\in\mathbb{P}^i$ such that $n\in P$, and similarly for $\mathbb{P}(n)$.
Then, by statement~\eqref{eq:1 step neighbour set} and Definition~\ref{df:function tau}, statement \eqref{eq:partition type in minimisation} can be reformulated as, for any integer $i\geq 2$,
\begin{equation}\label{eq:variety type in minimisation}
type^i(n)=\big\{\mathbb{P}^{i-1}(m)\mid m\in Adj_n\big\};
\end{equation}
\begin{equation}\label{eq:3-25-1}
type^i(n)=\big\{\mathbb{P}^{i-1}(\beta_n(b))\mid b\in \mathcal{B}_n\big\}.
\end{equation}
The next lemma follows directly from  statement~\eqref{eq:ith partition in minimisation}.
\begin{lemma}\label{lm:Mar-6-a}
\!For any integer $i\!\geq\!\! 2$ and any nodes $n,n'\!\in\! N$, the statement $\mathbb{P}^i(n)=\mathbb{P}^i(n')$is true if and only if the following statements are both true
\begin{itemize}
\item $\mathbb{P}^{i-1}(n)=\mathbb{P}^{i-1}(n')$;
\item $type^i(n)\!=\!type^i(n')$.
\end{itemize}
\end{lemma}

Then, the next lemma follows from statement~\eqref{eq:1st partition in minimisation} by induction on integer $i$.
\begin{lemma}\label{lm:Mar-6-b}
For any nodes $n,n'\in N$, if there is an integer $i\geq 1$ such that $\mathbb{P}^i(n)=\mathbb{P}^i(n')$, then $\ell_n=\ell_{n'}$.
\end{lemma}

The next lemma follows from the recursive definitions in statements~\eqref{eq:ith partition in minimisation} and \eqref{eq:partition type in minimisation}.
It shows that, once the sequence of partitions is stable, it will be stable forever.
\begin{lemma}\label{lm:Mar-8-a}
For any integer $i\geq 1$, if $\mathbb{P}^{i+1}=\mathbb{P}^{i}$, then $\mathbb{P}^{i+2}=\mathbb{P}^{i+1}$.
\end{lemma}
\begin{proof}
Consider two arbitrary nodes $n,n'$. In order to prove $\mathbb{P}^{i+2}=\mathbb{P}^{i+1}$, it suffices to show $\mathbb{P}^{i+2}(n)=\mathbb{P}^{i+2}(n')$ if and only if $\mathbb{P}^{i+1}(n)=\mathbb{P}^{i+1}(n')$.
The ``only if'' part of the statement follows from Lemma~\ref{lm:Mar-6-a}.
Next, we prove the ``if'' part, where we assume
\begin{equation}\label{eq:6-17-1}
\mathbb{P}^{i+1}(n)=\mathbb{P}^{i+1}(n').
\end{equation}

\begin{claim}\label{cl:Mar-8-a}
$type_n^{i+2}=type_{n'}^{i+2}$.
\end{claim}
\begin{proof-of-claim}
By the assumption $\mathbb{P}^{i+1}=\mathbb{P}^{i}$ of the lemma,
\begin{equation}\notag
\mathbb{P}^{i+1}(m)=\mathbb{P}^{i}(m)
\end{equation}
for each node $m\in N$.
Then, by statement~\eqref{eq:variety type in minimisation},
\begin{equation}\label{eq:Mar-8-1}
\begin{aligned}
type^{i+2}_n=&\big\{\mathbb{P}^{i+1}(m)\mid m\in Adj_n\big\}\\
=&\big\{\mathbb{P}^{i}(m)\mid m\in Adj_n\big\}=type^{i+1}_n;\\
type^{i+2}_{n'}=&\big\{\mathbb{P}^{i+1}(m)\mid m\in Adj_{n'}\big\}\\
=&\big\{\mathbb{P}^{i}(m)\mid m\in Adj_{n'}\big\}=type^{i+1}_{n'}.
\end{aligned}
\end{equation}
Note that $type^{i+1}_{n}=type^{i+1}_{n'}$ by Lemma~\ref{lm:Mar-6-a} and statement \eqref{eq:6-17-1}.
Thus, $type_n^{i+2}=type_{n'}^{i+2}$ by statement~\eqref{eq:Mar-8-1}.
\end{proof-of-claim}

Therefore, $\mathbb{P}^{i+2}(n)=\mathbb{P}^{i+2}(n')$ by Lemma~\ref{lm:Mar-6-a}, Claim~\ref{cl:Mar-8-a}, and statement~\eqref{eq:6-17-1}.
\end{proof}

The next corollary follows from Lemma~\ref{lm:Mar-8-a} and statement~\eqref{eq:finest partition in minimisation} by the standard definition of the limit of a sequence of sets \cite[Section 1.3]{resnick1998probability}.

\begin{corollary}\label{cr:3-12-a}
For any integer $i\!\geq\! 1$, if $\mathbb{P}^{i+1}\!=\!\mathbb{P}^{i}$, then $\mathbb{P}\!=\!\mathbb{P}^{i}$.
\end{corollary}

The above lemma shows that, to get the finest partition $\mathbb{P}$ in the sequence, we just need to find a partition $\mathbb{P}^i$ in the sequence such that $\mathbb{P}^{i+1}\!=\!\mathbb{P}^{i}$.
The next lemma shows that such an integer $i$ always exists and the sequence of partition stabilises after at most $|N|$ items.

\begin{lemma}\label{lm:finest partition exists}
There is an integer $i\leq |N|$ such that $\mathbb{P}=\mathbb{P}^j$ for each integer $j\geq i$.
\end{lemma}
\begin{proof}
By statement~\eqref{eq:ith partition in minimisation},
\begin{equation}\label{eq:Mar-8-3}
|\mathbb{P}^1|\leq |\mathbb{P}^2|\leq |\mathbb{P}^3|\leq \dots
\end{equation}
Note that, $|\mathbb{P}^i|\leq |N|$ for each $i\geq 1$ because $\mathbb{P}^i$ is a partition of set $N$.
Then, there must be an integer $i\leq |N|$ such that $|\mathbb{P}^{i+1}|=|\mathbb{P}^i|$ by statement~\eqref{eq:Mar-8-3}.
By statement~\eqref{eq:ith partition in minimisation}, this further implies the existence of an integer $i\leq |N|$ such that $\mathbb{P}^{i+1}=\mathbb{P}^i$. 
Then, the statement of the lemma follows from Corollary~\ref{cr:3-12-a} and Lemma~\ref{lm:Mar-8-a}.
\end{proof}

The next corollary extends the result in Lemma~\ref{lm:Mar-6-b} using the result in Lemma~\ref{lm:finest partition exists}.
It shows that, if two nodes are in the same set of the finest partition, then they must be labelled with the same agent.

\begin{corollary}\label{cr:same label in finest partition}
For any two nodes $n,n'\in N$, if $\mathbb{P}(n)=\mathbb{P}(n')$, then $\ell_n=\ell_{n'}$.
\end{corollary}

Up to now, we have seen some properties of the partition sequence that follow from the definitions in statements~\eqref{eq:1st partition in minimisation}, \eqref{eq:ith partition in minimisation}, and \eqref{eq:partition type in minimisation}.
In the rest of this subsection, we see some properties that reveal the reason why we consider such a partition sequence.

Recall that set $\Psi_n^i$ represents the $i$-length-bounded belief hierarchy of the (real or doxastic) agent denoted by node $n$ that consists of the belief sequences of length at most $i$.
The next lemma shows that the $i^{th}$ partition $\mathbb{P}^i$ in the sequence indeed partitions set $N$ according to the equivalence of set $\Psi_n^i$ of each node $n\in N$.
Informally speaking, two nodes are in the same set in partition $\mathbb{P}^i$ if and only if the agents they denote have the same $i$-length-bounded belief hierarchy.

\begin{lemma}\label{lm:partition same path set induction}
For any integer $i\geq 1$ and any nodes $n,n'\in N$, $\mathbb{P}^i(n)=\mathbb{P}^i(n')$ if and only if $\Psi_n^{i}=\Psi_{n'}^{i}$.
\end{lemma}
\begin{proof}
We prove the statement of the lemma by induction on integer $i$. 
We first consider the base case $i=1$.
Note that, $\mathbb{P}^1(n)=\mathbb{P}^1(n') \text{ if and only if } \ell_n=\ell_{n'}$ by statement~\eqref{eq:1st partition in minimisation}.
Meanwhile, $\Psi_n^{1}=\Psi_{n'}^{1} \text{ if and only if } \ell_n=\ell_{n'}$ by Definition~\ref{df:path set}.
Thus, $\mathbb{P}^1(n)=\mathbb{P}^1(n')$ if and only if $\Psi_n^{1}=\Psi_{n'}^{1}$.

Next, we consider the case $i\geq 2$.
For the ``if'' part of the statement, note that, by Lemma~\ref{lm:3-15-a} and the ``if'' part assumption $\Psi^{i}_{n}=\Psi^{i}_{n'}$,
\begin{equation}\label{eq:3-15-22}
\mathcal{B}_n=\mathcal{B}_{n'}
\end{equation}
and $\Psi^{i-1}_{\beta_{n}(b)}=\Psi^{i-1}_{\beta_{n'}(b)}$ for each agent $b\in\mathcal{B}_n$.
Then, by the induction hypothesis, for each agent $b\in\mathcal{B}_n$,
\begin{equation}\label{eq:3-18-1}
\mathbb{P}^{i-1}(\beta_n(b))=\mathbb{P}^{i-1}(\beta_{n'}(b)).
\end{equation}
Note that, by statement~\eqref{eq:3-25-1},
\begin{equation}\notag
\begin{aligned}
type^i(n)&=\big\{\mathbb{P}^{i-1}(\beta_n(b))\mid b\in\mathcal{B}_n\big\};\\
type^i(n')&=\big\{\mathbb{P}^{i-1}(\beta_{n'}(b))\mid b\in\mathcal{B}_{n'}\big\}.
\end{aligned}
\end{equation}
Hence, by statements~\eqref{eq:3-15-22} and \eqref{eq:3-18-1},
\begin{equation}\label{eq:3-11-5}
type^i(n)=type^i(n').
\end{equation}
Meanwhile, note that the ``if'' part assumption $\Psi^{i}_{n}=\Psi^{i}_{n'}$ implies that $\Psi^{i-1}_{n}=\Psi^{i-1}_{n'}$ by statement~\eqref{eq:accumulated path set} and Lemma~\ref{lm:equal path set}.
Then, $\mathbb{P}^{i-1}(n)=\mathbb{P}^{i-1}(n')$ by the induction hypothesis.
Therefore, $\mathbb{P}^{i}(n)=\mathbb{P}^{i}(n')$ by statement~\eqref{eq:3-11-5} and Lemma~\ref{lm:Mar-6-a}.

For the ``only if'' part of the statement, note that, the assumption $\mathbb{P}^i(n)=\mathbb{P}^i(n')$ implies
\begin{equation}\label{eq:Mar-6-5}
\ell_n=\ell_{n'}
\end{equation}
by Lemma~\ref{lm:Mar-6-b} and
\begin{gather}
\mathbb{P}^{i-1}(n)=\mathbb{P}^{i-1}(n'),\label{eq:Mar-6-6}\\
type^i(n)=type^i(n')\label{eq:Mar-6-7}
\end{gather}
by Lemma~\ref{lm:Mar-6-a}.
Then, $\Psi_n^{i-1}=\Psi_{n'}^{i-1}$ by statement~\eqref{eq:Mar-6-6} and the induction hypothesis.
Thus, to prove that $\Psi_n^{i}=\Psi_{n'}^{i}$, by statement~\eqref{eq:accumulated path set}, it suffices to prove that $\Pi_n^i=\Pi_{n'}^i$.

Note that, by statements~\eqref{eq:variety type in minimisation} and \eqref{eq:Mar-6-7},
\begin{equation}\notag
\big\{\mathbb{P}^{i-1}(m)\mid m\in Adj_n\big\}=\big\{\mathbb{P}^{i-1}(m')\mid m'\in Adj_{n'}\big\}.
\end{equation}
Then, by the induction hypothesis,
\begin{equation}\notag
\big\{\Psi_m^{i-1}\mid m\in Adj_n\big\}=\big\{\Psi_{m'}^{i-1}\mid m'\in Adj_{n'}\big\}.
\end{equation}
Thus, by Lemma~\ref{lm:equal path set},
\begin{equation}\notag
\big\{\Pi_m^{i-1}\mid m\in Adj_n\big\}=\big\{\Pi_{m'}^{i-1}\mid m'\in Adj_{n'}\big\}.
\end{equation}
Hence,
\begin{equation}\notag
\bigcup_{m\in Adj_n}\Pi_m^{i-1}=\bigcup_{m'\in Adj_{n'}}\Pi_{m'}^{i-1}.
\end{equation}
That is,
\begin{equation}\notag
\{\sigma\mid m\in Adj_n, \sigma\in\Pi_m^{i-1}\}\!=\!\{\sigma\mid m'\in Adj_{n'}, \sigma\in\Pi_{m'}^{i-1}\}.
\end{equation}
Therefore,
\begin{equation}\notag
\begin{aligned}
\Pi_n^{i}&=\{\ell_n\!::\!\sigma\mid m\in Adj_n, \sigma\in\Pi_{m}^{i-1}\}\\
&=\{\ell_{n'}\!::\!\sigma\mid m'\in Adj_{n'}, \sigma\in\Pi_{m'}^{i-1}\}=\Pi_{n'}^{i}
\end{aligned}
\end{equation}
by statement~\eqref{eq:Mar-6-5}.
\end{proof}

The next theorem extends the result in the above lemma to the limit of the partition sequence.
It shows that, in the finest partition $\mathbb{P}$, two nodes are in the set if and only if the agents they denote have the same belief hierarchy.

\begin{theorem}\label{th:finest partition same path set}
$\mathbb{P}(n)=\mathbb{P}(n')$ if and only if $\Psi_n^{*}=\Psi_{n'}^{*}$ for any nodes $n,n'\in N$.
\end{theorem}
\begin{proof}
If $\Psi_n^{*}=\Psi_{n'}^{*}$, then $\Psi_n^{i}=\Psi_{n'}^{i}$ for each $i\geq 1$ by Lemma~\ref{lm:equal path set} and Lemma~\ref{lm:equal whole path set}.
Thus, $\mathbb{P}^i(n)=\mathbb{P}^i(n')$ for each $i\geq 1$ by Lemma~\ref{lm:partition same path set induction}.
Hence, $\mathbb{P}(n)=\mathbb{P}(n')$ by Lemma~\ref{lm:finest partition exists}.

If $\Psi_n^{*}\neq\Psi_{n'}^{*}$, then $\Pi_n^{i}\neq\Pi_{n'}^{i}$ for some integer $i\geq 1$ by Lemma~\ref{lm:equal whole path set}.
Thus, $\Psi_n^{j}\neq\Psi_{n'}^{j}$ for each $j\geq i$ by Lemma~\ref{lm:equal path set}.
Hence, $\mathbb{P}^j(n)\neq\mathbb{P}^j(n')$ for each $j\geq i$ by Lemma~\ref{lm:partition same path set induction}.
Therefore, $\mathbb{P}(n)\neq\mathbb{P}(n')$ by Lemma~\ref{lm:finest partition exists}.
\end{proof}

From Theorem~\ref{th:finest partition same path set} and Theorem~\ref{th:doxastic equivalent nodes = indistinguishable}, we can see that, for any two nodes $n,n'\in N$, in the finest partition, $\mathbb{P}(n)=\mathbb{P}(n')$ if and only if they are doxastically equivalent.
It shows that the finest partition $\mathbb{P}$ corresponds to the doxastic equivalence relation on the set $N$ of nodes.

\subsection{Minimising an RBR Graph}\label{sec:app minimisation algorithm}

As mentioned in the main text, to minimise an RBR graph $B$, we compute an equivalent canonical RBR graph $B'$.
Recall that, set $\Psi_n^*$ captures the belief hierarchy of the agent denoted by node $n$.
Informally, by Theorem~\ref{th:node type matches in equivalent RBR graphs}, the nodes in the equivalent canonical RBR graph $B'$ should capture all the belief hierarchies in the RBR graph $B$.
Meanwhile, by Definition~\ref{df:canonical RBR graph}, different nodes in the equivalent canonical RBR graph $B'$ should capture different belief hierarchies.
On the other hand, for the original RBR graph $B$ and the finest partition $\mathbb{P}$ of its nodes set, by Theorem~\ref{th:finest partition same path set}, all nodes in the same set $P\in\mathbb{P}$ capture the same belief hierarchy (\textit{i.e.} doxastically equivalent) and different sets in partition $\mathbb{P}$ correspond to different belief hierarchies.
In this sense, we only need to let each node in the equivalent canonical RBR graph $B'$ correspond to a set $P\in\mathbb{P}$ and set the edges, labelling function, and the designating function accordingly.

The above is the intuition behind Algorithm~\ref{alg:compute CF}.
In a nutshell, for an input RBR graph, Algorithm~\ref{alg:compute CF} classifies the nodes into different classes according to the doxastic equivalence relation. This is done by recursively refining the partition of the set of nodes until the finest partition is received.
Each class of doxastically equivalent nodes in the input RBR graph corresponds to one node in the output RBR graph, the latter of which is doxastically equivalent to every node in the former.
The doxastic equivalence between nodes in the output and the input RBR graphs is achieved by making the edge set and the labelling function of the output RBR graph in line with those of the input RBR graph.
For the convenience of proof, we rewrite Algorithm~\ref{alg:compute CF} in the main text using the new notation $\mathbb{P}(\cdot ) $ defined in Appendix~\ref{sec:app partition sequence} as below.  
Differences appear in lines~\ref{algline:computing type}, \ref{algline:output edge set} and \ref{algline:output designation assign}.

\setcounter{algocf}{0}
\begin{algorithm}[hbt]
\footnotesize
\caption{Minimise an RBR graph}
\KwIn{RBR graph $(N,E,\ell,\pi)$}
\KwOut{RBR graph $(N', E',\ell',\pi')$}

$\mathbb{P}\leftarrow \big\{\{n'\in N\,|\,\ell_{n'}=\ell_n\}\,\big|\,n\in N\big\}$\label{algline: app initialise P}{\scriptsize\Comment*[r]{$\Psi_n^1$ equivalence}\small}
$stable\leftarrow false$\label{algline: app initialise flag}\;
\While({\scriptsize\Comment*[f]{$\Psi_n^i$ equiv. $\to\Psi_n^{i+1}$ equiv.}\small}){not $stable$\label{algline: app while start}}{
    $stable\leftarrow true$\;
    $\mathbb{P}'\leftarrow\varnothing$\label{algline: app initialise P'}\;
    \For{each set $P\in\mathbb{P}$\label{algline: app partition for loop start}}{
    \For{each node $n\in P$\label{algline: app type for loop start}}{
        $type(n)\leftarrow\{\mathbb{P}(n')\mid nEn'\}$\label{algline: app computing type}\;
    }
    $\mathbb{Q}\!\leftarrow\!\big\{\{n'\!\in\! P\mid type(n')\!=\!type(n)\}\mid n\!\in\! P\big\}$\label{algline: app partition each P}\;
    $\mathbb{P}'\leftarrow\mathbb{P}'\cup\mathbb{Q}$\label{algline: app update P'}\;
    \If{$|\mathbb{Q}|>1$\label{algline: app flag update if}}{
        $stable\leftarrow false$\label{algline: app partition for loop ends}\;
    }
    }
    $\mathbb{P}\leftarrow\mathbb{P}'$\label{algline: app while end}\;
}
$N'\leftarrow\mathbb{P}$\label{algline: app output node set}{\scriptsize\Comment*[r]{equivalent classes as nodes}\small}
$E'\leftarrow\{(\mathbb{P}(n),\mathbb{P}(m))\mid (n,m)\in E\}$\label{algline: app output edge set}\;
\For{each node $P\in N'$\label{algline: app output label for loop}}{
    pick an arbitrary node $n\in P$ and $\ell'_P\leftarrow\ell_{n}$ \label{algline: app output label assign}\;
}
\For{each agent $a\in\mathcal{A}$\label{algline: app output designation for loop}}{
    $\pi'_a\leftarrow \mathbb{P}(\pi_a)$ if $\pi_a$ is defined\label{algline: app output designation assign}\;
}
\Return{$(N', E',\ell',\pi')$}\label{algline: app CF algorithm return}\;
\end{algorithm}

Precisely speaking, in Algorithm~\ref{alg:compute CF}, line~\ref{algline:initialise P} initialises set $\mathbb{P}$ to be the partition $\mathbb{P}^1$ following from statement~\eqref{eq:1st partition in minimisation}.
Line~\ref{algline:initialise flag} initialises the variable \textit{stable} to track if the refining process reaches the finest partition or not.
Then, the partition $\mathbb{P}$ is refined in the \textbf{while} loop in lines~\ref{algline:while start} to \ref{algline:while end} following from statements~\eqref{eq:partition type in minimisation} and \eqref{eq:ith partition in minimisation}. 
In particular, in the $i^{th}$ iteration of the \textbf{while} loop, the set $\mathbb{P}$ equals partition $\mathbb{P}^i$, while the set $\mathbb{P}'$ is used to collect all the elements in partition $\mathbb{P}^{i+1}$.
Specifically, in each round of the \textbf{for} loop in lines~\ref{algline:partition for loop start} to \ref{algline:partition for loop ends}, 
the inner \textbf{for} loop in lines~\ref{algline:type for loop start} and \ref{algline:computing type} computes $type^{i+1}(n)$ for each node $n\in P$ following from statement~\eqref{eq:partition type in minimisation}.
Then, line~\ref{algline:partition each P} uses the \textit{type} information to compute the partition $\mathbb{Q}$ of set $P$, which is a subset of the partition $\mathbb{P}^{i+1}$ by statement~\eqref{eq:ith partition in minimisation} and thus combined to set $\mathbb{P}'$ in line~\ref{algline:update P'}.
For this reason, if $|\mathbb{Q}|>1$ (as in line~\ref{algline:flag update if}), then the refining process has not reached a stable state, and thus the \textbf{while} loop should continue.
On the other hand, if $|\mathbb{Q}|=1$ for each set $P\in\mathbb{P}$, then $\mathbb{P}=\mathbb{P}'$ and the \textbf{while} loop terminates. 
Hence, the next proposition follows from Corollary~\ref{cr:3-12-a}.

\begin{proposition}\label{pp:CF algorithm gets finest partition}
In Algorithm~\ref{alg:compute CF}, after the \textbf{while} loop in lines~\ref{algline:while start} to \ref{algline:while end} terminates, the set $\mathbb{P}$ is the finest partition of the set of nodes in the input RBR graph.
\end{proposition}

The rest of the algorithm generates the output RBR graph $(N',E',\ell',\pi')$ based on the finest partition $\mathbb{P}$.
In particular, each set $P\in\mathbb{P}$ is also a node\footnote{$P$ has the ``node-set duality'': it is both a set in the finest partition $\mathbb{P}$ and a node in the node set $N'$ of the output RBR graph. Thus, notation $\mathbb{P}(n)$ represents both the set containing node $n$ in the finest partition and the corresponding node in the output RBR graph.} in the output RBR graph (line~\ref{algline:output node set}).
For each edge $(n,m)\in E$ in the input RBR graph, an edge from the node $\mathbb{P}(n)$ to the node $\mathbb{P}(m)$ is generated in the output RBR graph (line~\ref{algline:output edge set}).
The label of each node $P$ in the output RBR graph is assigned the label of an arbitrary node $n\in P$ in the finest partition (lines~\ref{algline:output label for loop} and \ref{algline:output label assign}). 
Proposition~\ref{pp:same label in each P} below shows that the arbitrary nature of node $n$ does not affect the assignment result.
For each agent $a\in\mathcal{A}$, if $\pi_a$ is defined in the input RBR graph, then the designated node $\pi'_a$ in the output RBR graph is defined as the node $\mathbb{P}(\pi_a)$, which corresponds to the set containing node $\pi_a$ in the finest partition (lines~\ref{algline:output designation for loop} and \ref{algline:output designation assign}).

\begin{proposition}\label{pp:same label in each P}
$\ell'_P=\ell_n$ for each node $P\in N'$ and each node $n\in P$.
\end{proposition}
\begin{proof} 
Note that $\mathbb{P}(n_1)=\mathbb{P}(n_2)$ for any nodes $n_1,n_2\in P$.
Thus, $\ell_{n_1}=\ell_{n_2}$ for any nodes $n_1,n_2\in P$ by Corollary~\ref{cr:same label in finest partition}.
This means every node $n\in P$ is labelled with the same agent.
Hence, the statement of the proposition is true by line~\ref{algline:output label assign} of Algorithm~\ref{alg:compute CF}. 
\end{proof}

Note that line~\ref{algline:output label assign} is the only potentially nondeterministic step in Algorithm~\ref{alg:compute CF}.
With the result in Proposition~\ref{pp:same label in each P}, we show that Algorithm~\ref{alg:compute CF} is {\em deterministic}.
Next, we prove the correctness of Algorithm~\ref{alg:compute CF}, (\textit{i.e.} the output of Algorithm~\ref{alg:compute CF} is an equivalent canonical RBR graph of the input RBR graph).
In the next proposition, we prove that the output tuple $(N',E',\ell',\pi')$ is an RBR graph.

\begin{proposition}\label{pp:CF algorithm output is RBR graph}
The output tuple $(N',E',\ell',\pi')$ of Algorithm~\ref{alg:compute CF} is an RBR graph.
\end{proposition}
\begin{proof}
It suffices to prove that the tuple $(N',E',\ell',\pi')$ satisfies all four items of Definition~\ref{df:RBR graph}.
It is easily observed that \textbf{item~\ref{dfitem:RBR graph frame} of Definition~\ref{df:RBR graph}} is satisfied by lines~\ref{algline:output node set} and \ref{algline:output edge set} of Algorithm~\ref{alg:compute CF}.

Then, we show that \textbf{item~\ref{dfitem:RBR graph labelling function} of Definition~\ref{df:RBR graph}} is satisfied.
Note that $\mathbb{P}$ is a partition of the set of nodes in the input RBR graph.
Then, $P\neq\varnothing$ for each set $P\in\mathbb{P}$, which means a node $n\in P$ always exists.
Thus, $\ell'$ is a total function from set $N'$ to set $\mathcal{A}$ by lines~\ref{algline:output label for loop} and \ref{algline:output label assign} of Algorithm~\ref{alg:compute CF} because $\ell_n\in\mathcal{A}$.
Moreover, Claim~\ref{cl:3-13-1} and Claim~\ref{cl:3-13-2} below show that items~\ref{dfitem:RBR graph labelling function 1} and \ref{dfitem:different child label} of Definition~\ref{df:RBR graph} are satisfied, respectively.

\begin{claim}\label{cl:3-13-1}
If $(P,Q)\in E'$, then $\ell'_P\neq\ell'_Q$.
\end{claim}
\begin{proof-of-claim}
By line~\ref{algline:output edge set} of Algorithm~\ref{alg:compute CF} and the assumption $(P,Q)\in E'$ of the claim, there is a node $n\in P$ and a node $m\in Q$ such that $nEm$ in the input RBR graph.
Then, by item~\ref{dfitem:RBR graph labelling function 1} of definition~\ref{df:RBR graph},
\begin{equation}\label{eq:3-13-1}
\ell_n\neq\ell_m.
\end{equation}
Note that, $\ell_n=\ell'_P$ and $\ell_m=\ell'_Q$ by Proposition~\ref{pp:same label in each P} because $n\in P$ and $m\in Q$.
Thus, $\ell'_P\neq\ell'_Q$ by statement~\eqref{eq:3-13-1}.
\end{proof-of-claim}

\begin{claim}\label{cl:3-13-2}
If $(P,Q),(P,R)\in E'$ and $\ell'_Q=\ell'_R$, then $Q=R$.
\end{claim}
\begin{proof-of-claim}
By line~\ref{algline:output edge set} of Algorithm~\ref{alg:compute CF} and the assumption $(P,Q),(P,R)\in E'$, there are nodes $n_1,n_2\in P$, $m_1\in Q$, and $m_2\in R$ such that
\begin{equation}
(n_1,m_1)\in E \text{\hspace{1mm} and \hspace{1mm}} (n_2,m_2)\in E.\label{eq:3-13-4}
\end{equation}
Moreover, $\ell_{m_1}=\ell'_Q$ and $\ell_{m_2}=\ell'_R$ by Proposition~\ref{pp:same label in each P}.
Then, by the assumption $\ell'_Q=\ell'_R$ of the claim,
\begin{equation}\label{eq:3-13-5}
\ell_{m_1}=\ell_{m_2}.
\end{equation}
Note that $n_1,n_2\in P$, which means $\mathbb{P}(n_1)=\mathbb{P}(n_2)$, then 
\begin{equation}\label{eq:3-13-6}
\Psi^*_{n_1}=\Psi^*_{n_2}
\end{equation}
by Theorem~\ref{th:finest partition same path set}.
Moreover, $m_1\in Adj_{n_1}$ and $m_2\in Adj_{n_2}$ by statement~\eqref{eq:3-13-4}.
Then, $\Psi^{*}_{m_1}=\Psi^{*}_{m_2}$ by statements~\eqref{eq:3-13-5}, \eqref{eq:3-13-6}, and Lemma~\ref{lm:4-4-a}.
Hence, $\mathbb{P}(m_1)=\mathbb{P}(m_2)$ by Theorem~\ref{th:finest partition same path set}.
This means $Q=R$ because $m_1\in P$ and $m_2\in R$.
\end{proof-of-claim}

Next, we show that \textbf{item~\ref{dfitem:RBR graph designating function} of Definition~\ref{df:RBR graph}} is satisfied.
Note that $\pi_a\in\mathbb{P}(\pi_a)$.
Then, $\ell'_{\mathbb{P}(\pi_a)}=\ell_{\pi_a}$ by Proposition~\ref{pp:same label in each P}.
Thus, $\ell'_{\mathbb{P}(\pi_a)}=a$ by item~\ref{dfitem:RBR graph designating function} of Definition~\ref{df:RBR graph}.
Hence, $\ell'_{\pi'_a}=a$ by line~\ref{algline:output designation assign} of Algorithm~\ref{alg:compute CF}.

Finally, we show that \textbf{item~\ref{dfitem:RBR graph no irrelevant dummy node} of Definition~\ref{df:RBR graph}} is satisfied.
For any node $P\in N'$, consider an arbitrary node $n\in P$.
By item~\ref{dfitem:RBR graph no irrelevant dummy node} of Definition~\ref{df:RBR graph}, in the input RBR graph, there is a designated node $\pi_a\in N$ and a path $(m_1,\dots,m_k)$ where $k\geq 1$ such that
\begin{equation}\label{eq:3-14-1}
m_1=\pi_a,\;m_k=n,
\end{equation}
and $(m_i,m_{i+1})\in E$ for each integer $i<k$.
Then, by line~\ref{algline:output edge set} of Algorithm~\ref{alg:compute CF}, $(\mathbb{P}(m_i),\mathbb{P}(m_{i+1}))\in E'$ for each $i<k$.
Thus, a path $(\mathbb{P}(m_1),\dots,\mathbb{P}(m_k))$ exists.
Note that, $\mathbb{P}(m_1)=\mathbb{P}(\pi_a)=\pi'(a)$ and $\mathbb{P}(m_k)=\mathbb{P}(n)=P$ by statement~\eqref{eq:3-14-1}, line~\ref{algline:output designation assign} of Algorithm~\ref{alg:compute CF}, and the fact $n\in P$.
Hence, $(\mathbb{P}(m_1),\dots,\mathbb{P}(m_k))$ is a path from $\pi'_a$ to $P$. 
\end{proof}

Proposition~\ref{pp:CF algorithm output is RBR graph} above shows that the output of Algorithm~\ref{alg:compute CF} is indeed an RBR graph.
Next, we show the output RBR graph is equivalent to the input RBR graph by proving that the input RBR graph is locally isomorphic to the output RBR graph.
Particularly, for an input RBR graph $(N,E,\ell,\pi)$ Algorithm~\ref{alg:compute CF}, the finest partition $\mathbb{P}$ of set $N$, and the corresponding output RBR graph $(N',E',\ell',\pi')$, consider a function $\alpha:N\to N'$ such that, for each node $n\in N$,
\begin{equation}\label{eq:algorithm local isomorphism}
\alpha(n):=\mathbb{P}(n).
\end{equation}
\begin{proposition}\label{pp:CF algorithm surjection}
$\alpha$ is a surjection from set $N$ to set $N'$.
\end{proposition}
\begin{proof}
Note that $\mathbb{P}$ is a partition of set $N$. Then,
\begin{equation}\notag
\{\alpha(n)\mid n\in N\}=\{\mathbb{P}(n)\mid n\in N\}=\mathbb{P}=N'
\end{equation}
by statement~\eqref{eq:algorithm local isomorphism} and line~\ref{algline:output node set} of Algorithm~\ref{alg:compute CF}.
\end{proof}

\begin{proposition}\label{pp:CF algorithm edge match}
$\{\alpha(m)\mid nEm\}=\{Q\mid\alpha(n)E'Q\}$ for each node $n\in N$.
\end{proposition}
\begin{proof}
By statement~\eqref{eq:algorithm local isomorphism}, it suffices to prove
\begin{equation}\notag
\{\mathbb{P}(m)\mid nEm\}=\{Q\mid\mathbb{P}(n)E'Q\}.
\end{equation}

($\subseteq$):
Consider an arbitrary node $m\in N$ such that $nEm$. 
Then, $(\mathbb{P}(n),\mathbb{P}(m))\in E'$ by line~\ref{algline:output edge set} of Algorithm~\ref{alg:compute CF}.
Hence, $\mathbb{P}(m)\in\{Q\mid\mathbb{P}(n)E'Q\}$.

($\supseteq$):
Consider an arbitrary node $Q$ such that $\mathbb{P}(n)E'Q$. 
Then, by line~\ref{algline:output edge set} of Algorithm~\ref{alg:compute CF}, there are two nodes $n_1,m_1\in N$ such that
\begin{equation}\label{eq:3-14-6}
\mathbb{P}(n_1)=\mathbb{P}(n),\,\, \mathbb{P}(m_1)=Q, \text{ and } n_1Em_1.
\end{equation}
Note that, by statement~\eqref{eq:1 step neighbour set} and the third part of statement~\eqref{eq:3-14-6},
\begin{equation}\label{eq:3-14-8}
m_1\in Adj_{n_1},
\end{equation}
and, by Theorem~\ref{th:finest partition same path set} and the first part of statement~\eqref{eq:3-14-6},
\begin{equation}\label{eq:3-14-9}
\Psi^*_{n}=\Psi^*_{n_1}.
\end{equation}
Then, $\Psi^2_{n}\!=\!\Psi^2_{n_1}$ by Lemma~\ref{lm:equal path set} and Lemma~\ref{lm:equal whole path set}.
Thus, $\mathcal{B}_n\!=\!\mathcal{B}_{n_1}$ by Lemma~\ref{lm:3-15-a}.
Hence, by statements~\eqref{eq:1 step agent set} and \eqref{eq:3-14-8}, there is a node $m_2$ such that
\begin{equation}\label{eq:3-14-10}
m_2\in Adj_n \text{\hspace{1mm} and \hspace{1mm}} \ell_{m_2}=\ell_{m_1}.
\end{equation}
Note that, statements~\eqref{eq:3-14-8}, \eqref{eq:3-14-9}, and \eqref{eq:3-14-10} imply $\Psi^{*}_{m_2}=\Psi^{*}_{m_1}$ by Lemma~\ref{lm:4-4-a}.
Then, $\mathbb{P}(m_2)=\mathbb{P}(m_1)$ by Theorem~\ref{th:finest partition same path set}.
Thus, $Q=\mathbb{P}(m_2)$ by the second part of statement~\eqref{eq:3-14-6}.
Hence, $Q\in\{\mathbb{P}(m)\mid nEm\}$ by statement~\eqref{eq:1 step neighbour set} and the first part of statement~\eqref{eq:3-14-10}.
\end{proof}

\begin{proposition}\label{pp:CF algorithm label}
$\ell_n=\ell'_{\alpha(n)}$ for each node $n\in N$.
\end{proposition}
\begin{proof}
Recall that $n\in\mathbb{P}(n)$.
Then, $\ell_n=\ell'_{\mathbb{P}(n)}$ by Proposition~\ref{pp:same label in each P}.
Thus, $\ell_n=\ell'_{\alpha(n)}$ by statement~\eqref{eq:algorithm local isomorphism}.
\end{proof}

\begin{proposition}\label{pp:CF algorithm designation}
$\pi$ and $\pi'$ have the same domain $\mathcal{D}$ of definition and $\alpha(\pi_a)=\pi'_a$ for each agent $a\in\mathcal{D}$.
\end{proposition}
\begin{proof}
The statement of the proposition follows from statement~\eqref{eq:algorithm local isomorphism} and line~\ref{algline:output designation assign} of Algorithm~\ref{alg:compute CF}.
\end{proof}

\begin{proposition}\label{pp:CF output local isomorphic}
$\alpha$ is a local isomorphism from the input RBR graph $(N,E,\ell,\pi)$ to the output RBR graph $(N',E',\ell',\pi')$.
\end{proposition}
\begin{proof}
The statement of the proposition follows from Propositions~\ref{pp:CF algorithm surjection}-\ref{pp:CF algorithm designation} by Definition~\ref{df:RBR graph isomorphic}.
\end{proof}

Up to now, we have shown that the output of Algorithm~\ref{alg:compute CF} is an equivalent RBR graph of the input RBR graph. 
To prove the correctness of Algorithm~\ref{alg:compute CF}, we still need to prove the minimality of the output RBR graph.
By Lemma~\ref{lm:canonical is minimal}, it suffices to show that the output RBR graph of Algorithm~\ref{alg:compute CF} is canonical.
Then, by Definition~\ref{df:canonical RBR graph}, it suffices to prove the next proposition.

\begin{proposition}\label{pp:nodes not equivalent in output graph}
$\Psi^*_{P}\neq\Psi^*_{Q}$ for any distinct nodes $P,Q\in N'$.
\end{proposition}
\begin{proof}
Consider two arbitrary nodes $n\in P$ and $m\in Q$. Then, by statement~\eqref{eq:algorithm local isomorphism},
\begin{equation}\label{eq:3-15-4}
\alpha(n)=\mathbb{P}(n)=P \text{\hspace{1mm} and \hspace{1mm}} \alpha(m)=\mathbb{P}(m)=Q.
\end{equation}
Note that $\Psi^*_n=\Psi^*_{\alpha(n)}$ and $\Psi^*_m=\Psi^*_{\alpha(m)}$ by Lemma~\ref{lm:path set equal in isomorphic} and Proposition~\ref{pp:CF output local isomorphic}.
Then, by statement~\eqref{eq:3-15-4},
\begin{equation}\label{eq:3-15-5}
\Psi^*_n=\Psi^*_{P} \text{\hspace{1mm} and \hspace{1mm}} \Psi^*_m=\Psi^*_{Q}.
\end{equation}
Meanwhile, $\mathbb{P}(n)\neq\mathbb{P}(m)$ by statement~\eqref{eq:3-15-4} and the assumption that $P$ and $Q$ are distinct.
Then, $\Psi^*_n\neq\Psi^*_m$ by Theorem~\ref{th:finest partition same path set}.
Hence, $\Psi^*_{P}\neq\Psi^*_{Q}$ by statement~\eqref{eq:3-15-5}.
\end{proof}

As a summary, by Proposition~\ref{pp:CF algorithm output is RBR graph}, the output of Algorithm~\ref{alg:compute CF} is an RBR graph.
Then, by Lemma~\ref{lm:isomorphic equivalent} and Proposition~\ref{pp:CF output local isomorphic}, the output RBR graph is equivalent to the input RBR graph.
Thus, by Definition~\ref{df:canonical RBR graph} and Proposition~\ref{pp:nodes not equivalent in output graph}, the output RBR graph is an equivalent canonical RBR graph of the input RBR graph.
Hence, by Lemma~\ref{lm:canonical is minimal}, the output is the minimal equivalent RBR graph of the input RBR graph. 
We formally state this result in the next theorem.

\begin{theorem}\label{th:CF output canonical}
Algorithm~\ref{alg:compute CF} outputs the minimal equivalent RBR graph of the input RBR graph.
\end{theorem}

Recall that, by Theorem~\ref{th:equivalent canonical graph are isomorphic}, two equivalent canonical RBR graphs must be isomorphic.
Thus, we can conclude that  Algorithm~\ref{alg:compute CF} outputs the {\em unique} (up to isomorphism) minimal equivalent RBR graph of the input RBR graph.

\subsection{Time Complexity of Algorithm~\ref{alg:compute CF}}\label{sec:app complexity}

In this subsection, we analyse the time complexity of Algorithm~\ref{alg:compute CF}.
By $|\mathcal{A}|$ we mean the number of (both rational and irrational) agents in the system.
For an input RBR graph $(N,E,\ell,\pi)$, we use $|N|$ to denote the number of nodes.
To implement Algorithm~\ref{alg:compute CF}, we need some extra assumptions about the data structures of the set $\mathcal{A}$ of agents and an RBR graph $(N,E,\ell,\pi)$:
\begin{enumerate}[label={A\arabic*.}, ref=A\arabic*,left=0pt]
\item\label{item:6-10-4} All agents in set $\mathcal{A}$ are denoted by integers between $0$ and $|\mathcal{A}|-1$, \textit{i.e.}
\begin{equation}\notag
\mathcal{A}=\{0,\dots,|\mathcal{A}|-1\}.
\end{equation}
\item\label{item:6-10-3} All nodes in set $N$ are denoted by integers between $0$ and $|N|-1$, \textit{i.e.}
\begin{equation}\notag
N=\{0,\dots,|N|-1\}.
\end{equation}
\item\label{item:6-10-5} The edges are stored in a collection $\{E_n\}_{n\in N}$, where $E_n$ is an array of size $|\mathcal{A}|$ such that, for each agent (integer) $a\in\mathcal{A}$,
\begin{equation}\notag
E_n[a]=\begin{cases}
m\in N, &\exists m\in N (nEm \text{ and } \ell_m=a);\\
-1, &\text{otherwise}.
\end{cases}
\end{equation}
\item\label{item:6-10-1} The labelling function $\ell$ is stored in an array of size $N$ such that $\ell[n]=\ell_n\in\mathcal{A}$ for each node (integer) $n\in N$.
\item\label{item:6-10-2} The designating function $\pi$ is stored in an array of size $|\mathcal{A}|$ such that, for each agent (integer) $a\in\mathcal{A}$, 
\begin{equation}\notag
\pi[a]=\begin{cases}
\pi_a\in N, &\text{if $\pi_a$ is defined};\\
-1, &\text{otherwise}.
\end{cases}
\end{equation}
\end{enumerate}

Note that, the array $E_n$ in assumption~\ref{item:6-10-5} is well-defined by item~\ref{dfitem:RBR graph labelling function} of Definition~\ref{df:RBR graph}.
It is worth mentioning that, in both analysis and pseudocode of Algorithm~\ref{alg:compute CF}, we consider partitions as collections of sets.
However, in implementation, we use an array of size $|N|$ to denote a partition.
In particular, we index the sets in a partition $\mathbb{P}$ with integers from $0$ to $|\mathbb{P}|-1$ (\textit{i.e.} the number of sets in partition $\mathbb{P}$).
Then,
\begin{equation}\label{eq:6-17-2}
\mathbb{P}[n]=k\in\{0,\dots,|\mathbb{P}|-1\}
\end{equation}
if node $n$ is in the $(k+1)^{th}$ set in partition $\mathbb{P}$.

Moreover, given an array-formed partition $\mathbb{P}$, for each node $n\in N$, we define $Type_n$ to be an array of size $|\mathcal{A}|$ to collect the information of $type(n)$ in line~\ref{algline:computing type} of Algorithm~\ref{alg:compute CF} and the information of $\mathbb{P}(n)$ where node $n$ belongs in partition $\mathbb{P}$.
Specifically, for each node $n\in N$ and each integer $a\in\mathcal{A}$, the $a^{th}$ item in the array $Type_n$ is
\begin{equation}\label{eq:6-12-1}
Type_n[a]=\begin{cases}
\mathbb{P}[E_n[a]], & \text{if } E_n[a]\neq -1;\\
-\mathbb{P}[n], & \text{if } a=\ell_n;\\
|\mathbb{P}|, & \text{otherwise}.
\end{cases}
\end{equation}

Note that, by assumptions \ref{item:6-10-3}, \ref{item:6-10-5}, and item~\ref{dfitem:RBR graph labelling function 1} of Definition~\ref{df:RBR graph}, in the first case $E_n[a]\neq -1$ of statement \eqref{eq:6-12-1}, the notation $E_n[a]$ represents a node $m\neq n$ such that $nEm$ and $\ell_m=a$.
Then, the set $\{Type_n[a]\mid a\in\mathcal{A}, E_n[a]\neq -1\}$ is the set $type(n)$ in line~\ref{algline:computing type} in Algorithm~\ref{alg:compute CF}.
In other words, the array $Type_n$ contains the information of the set $type(n)$.
Meanwhile, in the second case of statement \eqref{eq:6-12-1} where $E_n[a]= -1$ and $a=\ell_n$, the $a^{th}$ item $Type_n[a]$ in the array equals $-\mathbb{P}(n)$.
On the one hand, it records the information of which set node $n$ belongs in partition $\mathbb{P}$.
On the other hand, the negation sign helps distinguish the information of $\mathbb{P}(n)$ from the information of $type(n)$.
Note that the integer $|\mathbb{P}|$ is {\em not} equal to $\mathbb{P}[m]$ for each node $m\in N$ by statement~\eqref{eq:6-17-2}.
Then, in the third case of statement \eqref{eq:6-12-1} where $E_n[a]= -1$ and $a\neq\ell_n$, the integer $|\mathbb{P}|$ serves as a {\em placeholder} for the rest items in the array $Type_n$.
Consequently, by Lemma~\ref{lm:Mar-6-a}, two arbitrary nodes $n,n'\in N$ belong to the same set in the refined partition $\mathbb{P}'$ if and only if $Type_n=Type_{n'}$.
As a result, in implementation of Algorithm~\ref{alg:compute CF}, we don't have to get the refined partition $\mathbb{P}'$ by partitioning each set $P\in\mathbb{P}$, as lines~\ref{algline:partition each P} and \ref{algline:update P'}  of Algorithm~\ref{alg:compute CF} do.
Instead, we can compute the array $Type_n$ of each node $n\in N$ and then partition set $N$ according to the equivalence of the array $Type_n$. Note that, the latter can be implemented by ranking the nodes according to the $Type_n$ arrays.

Based on the above analysis, we implement Algorithm~\ref{alg:compute CF} with the above data structures as Algorithm~\ref{alg:implementation} shows.
It is worth mentioning that, lines~\ref{algline:imp 17} and \ref{algline:imp 18} of Algorithm~\ref{alg:implementation} rank all nodes in set $N$ according to the $Type_n$ arrays.
This implementation is inspired by the radix sort algorithm.
The inner sort in line~\ref{algline:imp 18} takes $O(|N|\cdot log|N|)$ by a stable sorting algorithm such as merging sort.
Thus, one execution of the \textbf{for} loop in lines~\ref{algline:imp 17} and \ref{algline:imp 18} takes $O(|\mathcal{A}|\cdot|N|\cdot log|N|)$.
Moreover, the \textbf{while} loop starting at line~\ref{algline:imp 6} reflects the iterative refinement process.
Then, by Lemma~\ref{lm:finest partition exists}, it terminates after at most $|N|$ iterations.
Thus, during the whole process of Algorithm~\ref{alg:implementation}, the execution of lines~\ref{algline:imp 17} and \ref{algline:imp 18} takes $O(|\mathcal{A}|\cdot|N|^2\cdot log|N|)$ in total.
This forms the major part of the time complex of Algorithm~\ref{alg:implementation}.
The time complexity analysis of the other parts of Algorithm~\ref{alg:implementation} is trivial and thus omitted.
We mark the time complexity for the other parts as the notes in Algorithm~\ref{alg:implementation}.

In conclusion, with the specific assumptions of the data structures and the implementation in Algorithm~\ref{alg:implementation}, the RBR graph minimisation procedure takes $O(|\mathcal{A}|\cdot|N|^2\cdot log|N|)$ in time.

\begin{algorithm}[thb]
\footnotesize
\caption{Implementation of Algorithm~\ref{alg:compute CF}}
\label{alg:implementation}
\KwIn{$\mathcal{A}$, $N$, $\{E_n\}_{n\in N}$, $\ell$, $\pi$ in assumptions \ref{item:6-10-4}-\ref{item:6-10-2}}
\KwOut{RBR graph $(N',\{E'_P\}_{P\in N'},\ell',\pi')$}

initialise $\mathbb{P}$ and $\mathbb{P}'$ to be two arrays of size $|N|$\label{algline:imp 1}\;
\For({\scriptsize\Comment*[f]{$O(N)$}\small}){each $n\in N$\label{algline:imp 2}}{
    $\mathbb{P}[n]\leftarrow\ell_n$\;
    initialise $Type_n$ to be an array of size $|\mathcal{A}|$\;
}

$\mathbb{P}size\leftarrow |\mathcal{A}|$\;

\While({\scriptsize\Comment*[f]{at most $|N|$ loops}\small}){true\label{algline:imp 6}}{
    \For({\scriptsize\Comment*[f]{ $O(|N|)$}\small}){each $n\in N$\label{algline:imp type for loop start}}{
        \For({\scriptsize\Comment*[f]{$O(|\mathcal{A}|)$}\small}){each $a\in\mathcal{A}$}{
            \uIf{$E_n[a]\neq -1$}{
                 $Type_n[a]\leftarrow\mathbb{P}[E_n[a]]$\;
            }
            \uElseIf{$a=\ell[n]$}{
                $Type_n[a]\leftarrow-\mathbb{P}[n]$\;
            }
            \Else{
                $Type_n[a]\leftarrow size$\;
            }

        }   
    }
    $\vec{N}\leftarrow list(N)$\;
    \For({\scriptsize\Comment*[f]{$O(|\mathcal{A}|)$}\small}){each $a\in\mathcal{A}$\label{algline:imp 17}}{
        sort $\vec{N}$ by descending $Type_n[a]$\label{algline:imp 18}{\scriptsize\Comment*[r]{$O(|N|\!\cdot\!log|N||)$}\small}
    }
    $k\leftarrow 0$\;
    $\mathbb{P}'[\vec{N}[0]]\leftarrow k$\;
    \For({\scriptsize\Comment*[f]{$O(|N|)$}\small}){each $i$ from $1$ to $|N|-1$}{
        \If({\scriptsize\Comment*[f]{$O(|\mathcal{A}|)$}\small}){$type_{\vec{N}[i]}\neq type_{\vec{N}[i-1]}$}{
            $k\leftarrow k+1$\;
        }
        $\mathbb{P}'[\vec{N}[i]]\leftarrow k$\;
    }
    \If{$k+1=\mathbb{P}size$}{
        \textbf{break}\;
    }
    $\mathbb{P}\leftarrow\mathbb{P}'${\scriptsize\Comment*[r]{$O(|N|)$}\small}
    $\mathbb{P}size\leftarrow k+1$\;
}
$N'\leftarrow\{0,\dots,k\}$\;
\For({\scriptsize\Comment*[f]{$O(|N|)$}\small}){each $n\in N$}{
    \For({\scriptsize\Comment*[f]{$O(|\mathcal{A}|)$}\small}){each $a\in\mathcal{A}$}{
        \eIf{$E_n[a]= -1$}{
            $E'_{\mathbb{P}[n]}[a]\leftarrow-1$\;
        }{
            $E'_{\mathbb{P}[n]}[a]\leftarrow\mathbb{P}[E_n[a]]$\;
        }
        
    }
    $\ell'[\mathbb{P}[n]]\leftarrow\ell[n]$\;
}

\For({\scriptsize\Comment*[f]{$O(\mathcal{A}|)$}\small}){each agent $a\in\mathcal{A}$\label{algline:imp output designation for loop}}{
    \eIf{$\pi[a]=-1$}{
        $\pi'[a]\leftarrow -1$\;
    }{
        $\pi'[a]\leftarrow \mathbb{P}[\pi[a]]$\;
    }
}
\Return{$(N',\{E'_P\}_{P\in N'},\ell',\pi')$}\label{algline:imp CF algorithm return}\;
\end{algorithm}

\end{document}